\documentclass{article}
\usepackage{iclr2022_conference,times}
\iclrfinalcopy

\usepackage{amsmath,amsfonts,bm}

\def\eqref#1{equation~\ref{#1}}

\def\1{\bm{1}}

\def\rr{{\textnormal{r}}}

\def\va{{\bm{a}}}

\def\vc{{\bm{c}}}

\def\vf{{\bm{f}}}

\def\vv{{\bm{v}}}
\def\vw{{\bm{w}}}
\def\vx{{\bm{x}}}
\def\vy{{\bm{y}}}

\DeclareMathAlphabet{\mathsfit}{\encodingdefault}{\sfdefault}{m}{sl}
\SetMathAlphabet{\mathsfit}{bold}{\encodingdefault}{\sfdefault}{bx}{n}

\usepackage[utf8]{inputenc} 
\usepackage[T1]{fontenc}    
\usepackage{hyperref}       
\usepackage{url}            
\usepackage{booktabs}       
\usepackage{amsfonts}       
\usepackage{amsthm}
\usepackage{nicefrac}       
\usepackage{microtype}      
\usepackage{xcolor}         

\usepackage{caption}
\usepackage[flushleft]{threeparttable}

\usepackage{hyperref}
\usepackage{multirow}
\usepackage{wrapfig} 
\usepackage{times}
\usepackage{epsfig}
\usepackage{graphicx}
\usepackage{amsmath}
\usepackage{amssymb}
\usepackage[]{subfig}
\usepackage{algorithm}
\usepackage{algpseudocode}
\usepackage{tabularx}
\usepackage{enumitem}
\usepackage{amsmath}
\usepackage{booktabs}
\usepackage{makecell}

\def\dd{\mathrm{d}}

\usepackage{mathrsfs}
\usepackage{xcolor}

\usepackage{changes}
\usepackage{xspace}
\def\*#1{\mathbf{#1}}

\usepackage{subfig}
\usepackage{graphicx}
\usepackage{float}

\def\good{\mathrm{good}}
\def\bad{\mathrm{bad}}
\def\baseline{\mathrm{b}}
\def\circle#1{{\small \textcircled{\raisebox{-0.9pt}{#1}}}}

\title{Multi-objective Optimization by Learning Space Partitions}

\author{%
   Yiyang Zhao \\
  Worcester Polytechnic Institute \\
   \And
   Linnan Wang \\
  Brown University \\
   \And
   Kevin Yang \\
  UC Berkeley \\
   \And
  Tianjun Zhang \\
  UC Berkeley \\
   \AND
   Tian Guo \\
  Worcester Polytechnic Institute \\
   \And
   Yuandong Tian \\
  Facebook AI Research\\

}

\def\vx{\mathbf{x}}
\def\vy{\mathbf{y}}
\def\vf{\mathbf{f}}
\def\vc{\mathbf{c}}
\def\vw{\mathbf{w}}
\def\ours{LaMOO}
\def\rr{\mathbb{R}}

\def\qehvi{qEHVI}
\def\cmaes{CMA-ES}
\def\laqehvi{\ours{} with qEHVI}
\def\lacmaes{\ours{} with CMA-ES}

\newtheorem{lemma}{Lemma}
\newtheorem{theorem}{Theorem}
\newtheorem{definition}{Definition}
\newtheorem{observation}{Observation}

\iclrfinalcopy 

\begin{document}

\maketitle

\begin{abstract}

In contrast to single-objective optimization (SOO), multi-objective optimization (MOO) requires an optimizer to find the Pareto frontier, a subset of feasible solutions that are not dominated by other feasible solutions. In this paper, we propose \ours{}, a novel multi-objective optimizer that learns a model from observed samples to partition the search space and then focus on promising regions that are likely to contain a subset of the Pareto frontier. The partitioning is based on the \textit{dominance number}, which measures ``how close'' a data point is to the Pareto frontier among existing samples. To account for possible partition errors due to limited samples and model mismatch, we leverage Monte Carlo Tree Search (MCTS) to exploit promising regions while exploring suboptimal regions that may turn out to contain good solutions later. Theoretically, we prove the efficacy of learning space partitioning via \ours{} under certain assumptions. Empirically, on the HyperVolume (HV) benchmark, a popular MOO metric, LaMOO substantially outperforms strong baselines on multiple real-world MOO tasks, by up to 225\% in sample efficiency for neural architecture search on Nasbench201, and up to 10\% for molecular design.

\end{abstract}

\vspace{-0.1in}
\section{Introduction}
\vspace{-0.1in}
\label{sec:intro}
Multi-objective optimization (MOO) has been extensively used in many practical scenarios involving trade-offs between multiple objectives. For example, in automobile design~\citep{case1}, we must maximize the performance of the engine while simultaneously minimizing emissions and fuel consumption. In finance~\citep{case2}, one prefers a portfolio that maximizes the expected return while minimizing risk.

Mathematically, in MOO we optimize $M$ objectives $\vf(\vx) = [f_1(\vx), f_2(\vx), \ldots, f_M(\vx)]\in \rr^M$:
\begin{eqnarray}
\min\;& f_1(\vx), f_2(\vx), ..., f_M(\vx)  \label{eq:problem-setting}  \label{prob-formulation} \\
\mathrm{s.t.}\;& \vx \in \Omega \nonumber
\end{eqnarray}
While we could set arbitrary weights for each objective to turn it into a single-objective optimization (SOO) problem, modern MOO methods aim to find the problem's entire \emph{Pareto frontier}: the set of solutions that are not \emph{dominated} by any other feasible solutions\footnote{Here we define \emph{dominance} $\vy \prec_\vf \vx$ as $f_i(\vx) \leq f_i(\vy)$ for all functions $f_i$, and exists at least one $i$ s.t. $f_i(\vx) < f_i(\vy)$, $1\le i \le M$. That is, solution $\vx$ is always better than solution $\vy$, regardless of how the $M$ objectives are weighted.} (see Fig.~\ref{fig:concepts} for illustration). The Pareto frontier yields a global picture of optimal solution structures rather than focusing on one specific weighted combination of objectives.

As a result, MOO is fundamentally different from SOO. Instead of focusing on a single optimal solution, a strong MOO optimizer should cover the search space broadly to explore the Pareto frontier. Popular quality indicators in MOO, such as hypervolume (HV), capture this aspect by computing the volume of the currently estimated frontier. Specifically, given a reference point $R \in \rr^M$, as shown in Fig.~\ref{fig:concepts}(a), the \emph{hypervolume} of a finite approximate Pareto set $\mathcal{P}$ is the M-dimensional Lebesgue measure $\lambda_{M}$ of the space dominated by $\mathcal{P}$ and bounded from below by $R$. That is, $HV(\mathcal{P}, R) = \lambda_{M} (\cup_{i=1}^{|\mathcal{P}|}[R, y_{i}])$, where $[R, y_{i}]$ denotes the hyper-rectangle bounded by reference point $R$ and $y_{i}$. Consequently, the optimizer must consider the diversity of solutions in addition to their optimality.

While several previous works have proposed approaches to capture this diversity-optimality trade-off ~\citep{nsga-ii, nsgaiii, cmaes_moo, qehvi, qparego}, in this paper, we take a fundamentally different route by \emph{learning} promising candidate regions from past explored samples. Ideally, to find the Pareto frontier in as few function evaluations as possible, we want to sample heavily in the Pareto optimal set $\Omega_P$, defined as the region of input vectors that corresponds to the Pareto frontier.

One way to focus samples on $\Omega_P$ is to gradually narrow the full search space down to the subregion containing $\Omega_P$ via partitioning. For example, in the case of quadratic objective functions, $\Omega_P$ can be separated from the non-optimal set $\Omega \backslash \Omega_P$ via simple linear classifiers (see Observation~\ref{obs:isotropic},\ref{obs:cutting-planes}). Motivated by these observations, we thus design \ours{}, a novel MOO meta-optimizer that progressively partitions regions into sub-regions and then focuses on sub-regions that are likely to contain Pareto-optimal regions, where existing solvers can help. Therefore, \ours{} is a meta-algorithm.  

Unlike cutting-plane methods~\citep{loganathan1987convergent, vieira2019cutting, hinder2018cutting} that leverage the (sub)-gradient of convex objectives as the cutting plane, with global optimality guarantees, \ours{} is data-driven: it leverages previous samples to build classifiers to \textit{learn} the partition and focuses future samples in these promising regions. No analytical formula of objectives or their sub-gradients is needed. \ours{} is a multi-objective extension of recent works ~\citep{wang2020learning, plalam} that also learn space partitions but for a single black-box objective.

Empirically, \ours{} outperforms existing approaches on many benchmarks, including standard benchmarks in multi-objective black-box optimization, and real-world multi-objective problems like neural architecture search (NAS)~\citep{ofa, proxylessnas} and molecule design. For example, as a meta-algorithm, \ours{} combined with CMA-ES as an inner routine requires only 62.5\%, 8\%, and 29\% as many samples to reach the same hypervolume as the original CMA-ES~\citep{cma-es} in BraninCurrin~\citep{bc_func}, VehicleSafety~\citep{vehicle_safety} and Nasbench201~\citep{nasbench201}, respectively. On average, compared to qEHVI, \ours{} uses 50\% samples to achieve the same performance in these problems. In addition, \ours{} with qEHVI~\citep{qehvi} and CMA-ES require 71\% and 31\% fewer samples on average, compared to naive qEHVI and CMA-ES, to achieve the same performance in molecule discovery. Our code can be found at \url{https://github.com/aoiang/LaMOO}.

\vspace{-0.1in}
\section{Related Work}
\vspace{-0.1in}
\label{related-works}

\begin{figure}
	\begin{minipage}{0.43\linewidth}
		\includegraphics[width=58mm]{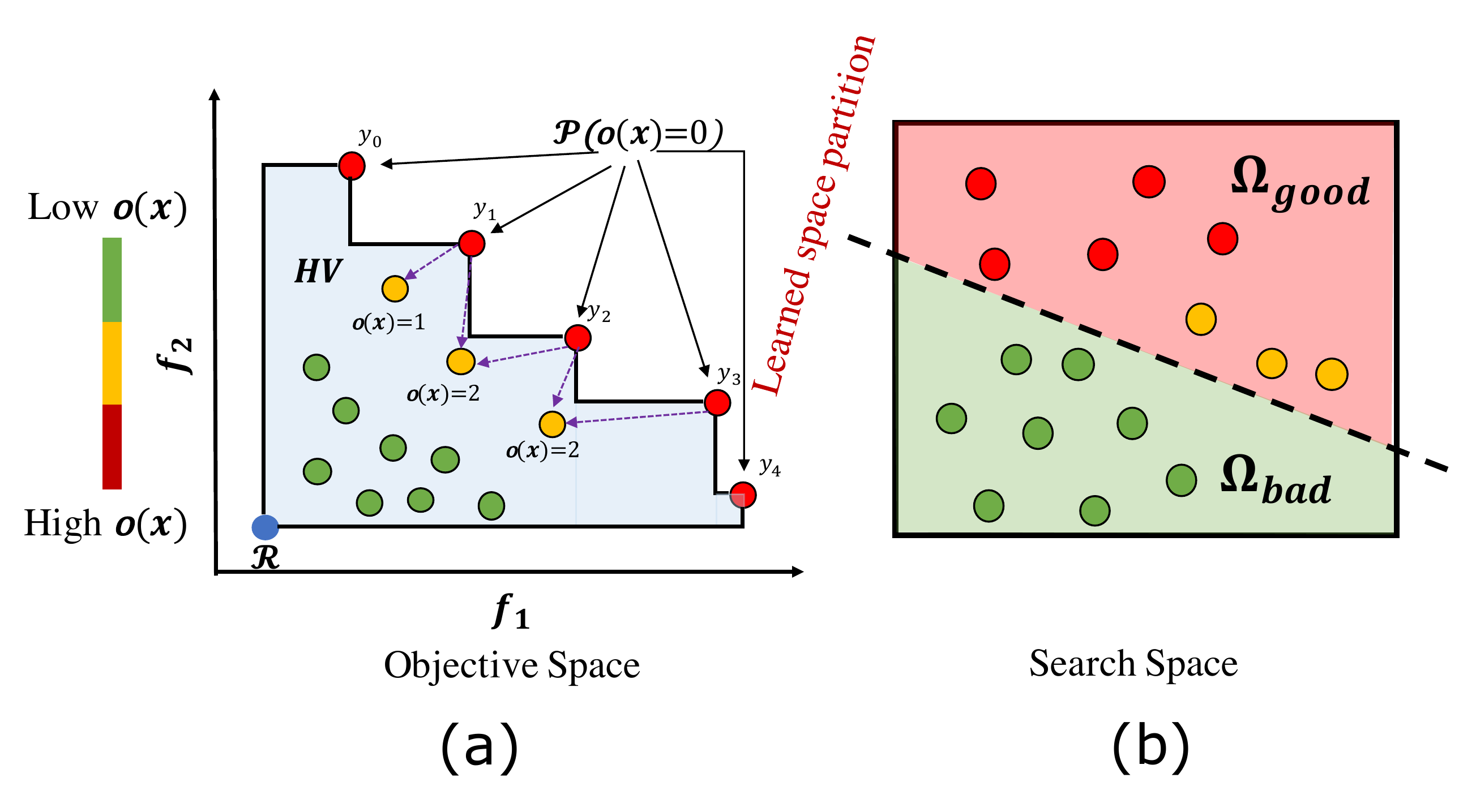}

	\end{minipage}\hfill	
	\begin{minipage}{0.58\linewidth}
		\scriptsize

        \begin{tabular}{|c|c|c|}
            \hline
            \textbf{MOO methods} & \textbf{Sampling Method}                     & \textbf{Objectives\textgreater 3} \\ \hline
            MOEA/D~\citep{moead}      & \multirow{4}{*}{Evolution}             & $\times$                         \\ \cline{3-3} 
            CMA-ES~\citep{cma-es}      &                                        & $\times$                         \\ \cline{3-3} 
            NSGA-II~\citep{nsga-ii}     &                                        & $\times$                         \\ \cline{3-3} 
            NAGA-III~\citep{nsgaiii}    &                                        & $\surd$                          \\ \hline
            qPAREGO~\citep{qparego}     & \multirow{2}{*}{Bayesian optimization} & $\surd$                          \\ \cline{3-3} 
            qEHVI~\citep{qehvi}       &                                        & $\surd$                          \\ \hline
            \textbf{LaMOO} (our approach)     & Space partition                        & $\surd$                          \\ \hline
            \end{tabular}
	\end{minipage}
	\vspace{-0.2in}
	\caption{\small \textbf{Left}: A basic setting in Multi-objective Optimization (MOO), optimizing $M=2$ objectives in Eqn.~\ref{prob-formulation}. (a) depicts the objective space $(f_1, f_2)$ and (b) shows the search space $\vx \in \Omega$. In (a), $P$ denotes the Pareto frontier, $R$ is the reference point, the hypervolume $HV$ is the space of the shaded area, and $o(\vx)$ are the dominance numbers. In (b), once a few samples are collected within $\Omega$, \ours{} \emph{learns to partition} the search space $\Omega$ into sub-regions (i.e. $\Omega_{good}$ and $\Omega_{bad}$) according to the dominance number in objective space, and then focuses future sampling on the good regions that are close to the Pareto Frontier. This procedure can be repeated to further partition $\Omega_{good}$ and $\Omega_{bad}$. \textbf{Right}: A table shows the properties of MOO methods used in experiments.}
\label{fig:concepts}
\end{figure}

\textbf{Bayesian Optimization (BO)}~\citep{qparego, qehvi, 1197687, YANG2019945, couckuyt2014fast, thompson_sample, sms_ego} is a popular family of methods to optimize black-box single and multi-objectives. Using observed samples, BO learns a surrogate model $\hat f(\vx)$, search for new promising candidates based on \emph{acquisition function} built on $\hat f(\vx)$, and query the quality of these candidates with the ground truth black-box objective(s). In multi-objective Bayesian optimization (MOBO), most approaches leverage Expected Hypervolume Improvement (EHVI) as their acquisition function~\citep{1197687, YANG2019945, couckuyt2014fast}, since finding the Pareto frontier is equivalent to maximizing the hypervolume given a finite search space~\citep{hv_prove}. 
There are methods~\citep{qparego, thompson_sample, sms_ego} that use different acquisition functions like expected improvement~\citep{jones1998efficient} and Thompson sampling~\citep{thompson1933likelihood}. 
EVHI is computationally expensive: its cost increases exponentially with the number of objectives. To address this problem, qEHVI~\citep{qehvi} accelerates optimization by computing EHVI in parallel, and has become the state-of-the-art MOBO algorithm. In this paper, we leverage qEHVI as a candidate inner solver in our proposed \ours{} algorithm. 

\textbf{Evolutionary algorithms (EAs)}~\citep{nsga-ii, improved_nsga, moead, SMS-EMOA,cma-es} are also popular methods for MOO tasks. MOO-EAs can be roughly categorized into three groups. One category~\citep{nsga, nsga-ii, nsgaiii} leverages Pareto dominance to simultaneously optimize all objectives. A second category (e.g.,~\citep{moead}) decomposes a multi-objective optimization problem into a number of single-objective sub-problems, converting a difficult MOO into several SOOs. Another category is quality indicator-based methods, such as ~\citep{SMS-EMOA} and ~\citep{cma-es}. They scalarize the current Pareto frontier using quality indicators (e.g., HV) and transfer a MOO to a SOO. New samples are generated by crossover and mutation operations from existing ones. While EAs have shown success in two-objective problems, non-quality indicator-based methods (i.e., the first two categories) struggle to handle a larger number of objectives. Specifically, for MOO with many objectives, NSGA-II~\citep{nsga-ii} easily gets stuck in a dominance resistant solution~\citep{improved_nsgaii} which is far from the true Pareto frontier. Determining the number of weight vectors, and efficiently generating and initializing them, are the main challenges of MOEA/D in many-objective ($>3$) problems~\citep{moead, moead_survey}.

\textbf{Quality Indicators}. Besides hypervolume, there are several other quality indicators~\citep{gd, igd, maxspread, spacing, error_ratio} for evaluating sample quality, which can be used to scalarize the MOO to SOO.  The performance of a quality indicator can be evaluated by three metrics~\citep{QI_res1, QI_res2}, including convergence (closeness to the Pareto frontier), uniformity (the extent of the samples satisfy the uniform distribution), and spread (the extent of the obtained approximate Pareto frontier). Sec.~\ref{app:qi} specifically illustrates the merits of each quality indicator. 
HyperVolume is the only metric we explored that can simultaneously satisfy the evaluation of convergence, uniformity, and spread without the knowledge of the true Pareto frontier. Therefore, throughout this work, we use HV to evaluate the optimization performance of different algorithms.

\vspace{-0.1in}
\section{Learning Space Partitions: A Theoretical Understanding}
\vspace{-0.1in}
\label{sec:theoretical-understanding}
Searching in high-dimensional space to find the optimal solution to a function is in general a challenging problem, especially when the function's properties are unknown to the search algorithm. The difficulty is mainly due to the curse of dimensionality: to adequately cover a $d$-dimensional space, in general, an exponential number of samples are needed.  

For this, many works use a ``coarse-to-fine'' approach: partition the search space and then focusing on promising regions. Traditionally, manually defined criteria are used, e.g., axis-aligned partitions~\citep{munos2011optimistic}, Voronoi diagrams~\citep{kim2020monte}, etc. Recently,~\citep{lanas,wang2020learning,plalam} \emph{learn} space partitions based on the data collected thus far, and show strong performance in NeurIPS black box optimization challenges~\citep{comp_4th, comp_9th}. 

On the other hand, there is little quantitative understanding of space partition. In this paper, we first give a formal theoretical analysis on why learning plays an important role in space-partition approaches for SOO. Leveraging our understanding of how space partitioning works, we propose \ours{} which empirically outperforms existing SoTA methods on multiple MOO benchmarks. 

\vspace{-0.1in}
\subsection{Problem Setting} 
\vspace{-0.1in}
Intuitively, learning space partitions will yield strong performance if the classifier can determine which regions are promising given few data points. We formalize this intuition below and show why it is better than fixed and manually defined criteria for space partitioning.  

Consider the following sequential decision task. We have $N$ samples in a \emph{discrete} subset $S_0$ and there exists one sample $\vx^*$ that achieves a minimal value of a scalar function $f$. Note that $f$ can be any property we want, e.g., in the Pareto optimal set. The goal is to construct a subset $S_T \subseteq S_0$ after $T$ steps, so that (1) $\vx^* \in S_T$ and (2) $|S_T|$ is as small as possible. More formally, we define the reward function $r$ as the probability that we get $\vx^*$ by randomly sampling from the resulting subset $S_T$:
\begin{equation}
    r := \frac{1}{|S_T|}P(\vx^* \in S_T) 
\end{equation}
It is clear that $0 \le r \le 1$. $r = 1$ means that we already found the optimal sample $\vx^*$. 

Here we use discrete case for simplicity and leave continuous case (i.e., partitioning a region $\Omega_0$ instead of a discrete set $S_0$) to future work. Note $N$ could be large, so here we consider it infeasible to enumerate $S_0$ to find $\vx^*$. However, sampling from $S_0$, as well as comparing the quality of sampled solutions are allowed. An obvious baseline is to simply set $S_T := S_0$, then $r_\baseline = N^{-1}$. Now the question is: can we do better? Here we seek help from the following \emph{oracle}:
\begin{definition}[$(\alpha, \eta)$-Oracle] 
Given a subset $S$ that contains $\vx^*$, after taking $k$ samples from $S$, the oracle can find a \emph{good} subset $S_\good$ with $|S_\good| \le |S| / 2$ and 
\begin{equation}
    P\left(\vx^* \in S_\good | \vx^* \in S\right) \ge 1 - \exp\left(- \frac{k}{\eta |S|^\alpha}\right)
\end{equation}
\end{definition}
\begin{lemma}
The algorithm to uniformly draw $k$ samples in $S$, pick the best and return is a $(1,1)$-oracle.  
\end{lemma}
See Appendix for proof. Note that a $(1, 1)$-oracle is very weak, and is of little use in obtaining higher reward $r$. We typically hope for an oracle with smaller $\alpha$ and $\eta$ (i.e., both smaller than 1). Intuitively, such oracles are more sample-efficient: with few samples, they can narrow down the region containing the optimal solution $\vx^*$ with high probability.

Note that $\alpha < 1$ corresponds to \emph{semi-parametric models}. In these cases, the oracle has \emph{generalization property}: with substantially fewer samples than $N$ (i.e., on the order of $N^\alpha$), the oracle is able to put the optimal solution $\vx^*$ on the right side. In its extreme case when $\alpha=0$ (or \emph{parametric models}), whether we classify the optimal solution $\vx^*$ on the correct side only depends on the \emph{absolute} number of samples collected in $S$, and is independent of its size. For example, if the function to be optimized is linear, then with $d+1$ samples, we can completely characterize the property of all $|S|$ samples. 

\textbf{Relation with cutting plane.} Our setting can be regarded as a data-driven extension of cutting plane methods~\citep{loganathan1987convergent, vieira2019cutting, hinder2018cutting} in optimization, in which a cutting plane is found at the current solution to reduce the search space. For example, if $f$ is convex and its gradient $\nabla f(\vx)$ is available, then we can set $S_\good := \{\vx: \nabla f(\vx_0)^\top(\vx - \vx_0) \le 0, \vx \in S_0\}$, since for any $\vx \in S_0 \setminus S_\good$, convexity gives $f(\vx) \ge f(\vx_0) + \nabla f(\vx_0)^\top(\vx - \vx_0) > f(\vx_0)$ and thus $\vx$ is not better than current $\vx_0$. 

However, the cutting plane method relies on certain function properties like convexity. In contrast, learning space partition can leverage knowledge about the function forms, combined with observed samples so far, to better partition the space.

\vspace{-0.1in}
\subsection{Rewards under optimal action sequence} 
\vspace{-0.1in}
We now consider applying the $(\alpha, \eta)$-oracle iteratively for $T$ steps, by drawing $k_t$ samples from $S_{t-1}$ and setting $S_t := S_{\good,t-1}$. We assume a total sample budget $K$, so $\sum_{t=1}^T k_t = K$. Note that $T \le \log_2 N$ since we halve the set size with each iteration.

Now the question is twofold. (1) How can we determine the action sequences $\{k_t\}$ in order to maximize the total reward $r$? (2) Following the optimal action sequences $\{k^*_t\}$, can $r^*$ be better than the baseline $r_\baseline = N^{-1}$? The answer is yes.

\begin{theorem}
\label{thm:reward}
The algorithm yields a reward $r^*$ lower bounded by the following:
\begin{equation}
    r^* \ge r_\baseline \exp\left[\left(\log 2 - \frac{\eta N^\alpha \phi(\alpha, T)}{K}\right)T\right] 
\end{equation}
where $r_\baseline := N^{-1}$ and $\phi(\alpha, T) := (1 - 2^{-\alpha T}) / (1 - 2^{-\alpha})$.
\end{theorem}
\textbf{Remarks.} Following Theorem~\ref{thm:reward}, a key condition to make $r^* > r_\baseline$ is to ensure $\log 2 > \frac{\eta N^\alpha \phi(\alpha, T)}{K}$. This holds if when $\frac{\eta N^\alpha \phi(\alpha, T)}{K} \rightarrow 0$. Note that since $T \le \log_2 N$, the final reward $r^*$ is upper bounded by 1 (rather than goes to $+\infty$). We consider some common practical scenarios below.

\emph{Non-parametric models ($\alpha=1$).} In this case, $\phi(\alpha, T) \le 2$  and the condition becomes $\frac{1}{2}\log 2 > \eta N / K$. This happens when the total sample budget $K = \Theta(N)$, i.e., on the same order of $N$, which means that the partitioning algorithm obtains little advantage over exhaustive search.

\emph{Semi-parametric models ($\alpha<1$).} In this case, $\phi(\alpha, T) \le 1 / (1 - 2^{-\alpha})$ and the condition becomes $(1 - 2^{-\alpha})\log 2 > \eta N^\alpha / K$. This happens when the total sample budget $K = \Theta(N^\alpha)$. In this case, we could use many fewer samples than exhaustive search to achieve better reward, thanks to the generalization property of the oracle. 

\emph{Parametric models ($\alpha=0$).} Now $\phi(\alpha, T) = T$ and the condition becomes $\log 2 > \frac{\eta T}{K}$. Since $T \le \log_2 N$, the total sample budget can be set to be $K = \Theta(\log N)$. Intuitively, the algorithm performs iterative halving (or binary search) to narrow down the search toward promising regions.

\label{sec:method}
\subsection{Extension to Multi-Objective Optimization}
\vspace{-0.1in}
Given our understanding of space partitioning, we now extend this idea to MOO. Intuitively, we want ``good'' regions to be always picked by the space partition. For SOO, it is possible since the optimal solution is a single point. How about MOO?

Unlike SOO, in MOO we aim for a continuous region, the \emph{Pareto optimal set} $\Omega_P := \{\vx: \nexists \vx'\neq \vx: \ \vf(\vx') \prec \vf(\vx)\}$. A key variable is the regularity of $\Omega_P$: if it is highly non-regular and not captured by a simple partition boundary (ideally a parametric boundary), then learning a space partition would be difficult. Interestingly, the shape of $\Omega_P$ can be characterized for quadratic objectives:
\begin{observation}
\label{obs:isotropic}
If all $f_j$ are isotropic, $f_j(\vx) = \|\vx - \vc_j\|^2_2$, then $\Omega_P = \mathrm{ConvexHull}(\vc_1, \ldots, \vc_q)$.
\end{observation}

\begin{observation}
\label{obs:cutting-planes}
If $M = 2$ and $f_j(\vx) = (\vx - \vc_j)^\top H_j (\vx-\vc_j)$ where $H_j$ are positive definite symmetric matrices, then there exists $\vw_1 := H_2(\vc_2 - \vc_1)$ and $\vw_2:= H_1(\vc_1-\vc_2)$, so that for any $\vx \in \Omega_P$, $\vw_1^\top(\vx - \vc_1) \ge 0$ and $\vw_2^\top(\vx - \vc_2) \ge 0$. 
\end{observation}
In both cases, $\Omega_P$ can be separated from non-Pareto regions $\Omega\backslash \Omega_P$ via a linear hyperplane. Empirically, $\Omega_P$ only occupies a small region of the entire search space (Sec.~\ref{sec:procedure}), and quickly focusing samples on the promising regions is critical for high sample efficiency. 

In the general case, characterizing $\Omega_P$ is analytically hard and requires domain knowledge about the objectives~\citep{QI_res2}. However, for MOO algorithms in practice, knowing that $\Omega_P$ can be separated from $\Omega\backslash\Omega_P$ via simple decision planes is already useful: we could learn such decision planes given previous data that are already collected, and sample further in promising regions.

\vspace{-0.1in}
\section{\ours{}: Latent Action Multi-objective Optimization}
\vspace{-0.1in}
\label{sec:procedure}
In Sec.~\ref{sec:theoretical-understanding}, for convenience, we only analyze a greedy approach, which makes decisions on space partitions and never revises them afterwards. While this greedy approach indeed works (as shown in Sec.~\ref{sec:ablation_hyper_params}), an early incorrect partition could easily rule out regions that turn out to be good but weren't identified with few samples. 
In practice, we want to keep the decision softer: while \emph{exploiting} the promising region, we also \emph{explore} regions that are currently believed to be sub-optimal given limited samples. It is possible that these regions turn out to contain good solutions when more samples are available, and the oracle can then make a different partition.  

To balance the trade-off between exploration and exploitation to cope with the generalization error of the learned classifier, we leverage Monte Carlo Tree Search (MCTS)~\citep{mcts} and propose our algorithm \ours{}. As shown in Alg.~\ref{alg:lamoo}, \ours{} has four steps: (1) learn to partition the search space given previous observed data points $D_{t}$, which are collected \{$\mathbf{x}_i$, f($\mathbf{x}_i$)\} from iterations $0$ to $t$. (2) With this information, we partition the region into promising and non-promising regions, and learn a classifier $h(\cdot)$ to separate them. (3) We select the region to sample from, based on the UCB value of each node. (4) We sample selected regions to obtain future data points $D_{t+1}$.

\algdef{SE}[SUBALG]{Indent}{EndIndent}{}{\algorithmicend\ }%
\algtext*{Indent}
\algtext*{EndIndent}


\def\root{\mathrm{root}}

\begin{algorithm*}[t]
    \small
	\caption{\ours{} Pseudocode.}
	\begin{algorithmic}[1]
	\State {\bfseries Inputs:} Initial $D_0$ from uniform sampling, sample budget $T$.
	\For{$t = 0, \dots, T$}
	\State Set $\mathcal{L} \leftarrow \{\Omega_\root\}$ (collections of regions to be split). 
	\While{$\mathcal{L} \neq \emptyset$}
	\State $\Omega_j \leftarrow \mathrm{pop\_first\_element}(\mathcal{L}),\ \  D_{t,j} \leftarrow D_t \cap \Omega_j, \ \ n_{t,j} \leftarrow |D_{t,j}|$. 
	\State Compute dominance number $o_{t,j}$ of $D_{t,j}$ using Eqn.~\ref{eq:dominance} and train SVM model $h(\cdot)$.
	\State \textbf{If} $(D_{t,j}, o_{t,j})$ is splittable by SVM, \textbf{then} $\mathcal{L} \leftarrow \mathcal{L} \cup \mathrm{Partition}(\Omega_j, {h(\cdot)})$.
	\EndWhile
	\For{$k = \root$, $k$ is not leaf node}
	    \State $D_{t,k} \leftarrow D_t \cap \Omega_k, \ \ v_{t,k} \leftarrow \mathrm{HyperVolume}(D_{t,k}),\ \ n_{t,k} \leftarrow$ $|D_{t,k}|$.
	    \State $k \leftarrow \displaystyle\arg\max_{c\ \in \ \mathrm{children}(k)} \mathrm{UCB}_{t,c}$, where $\mathrm{UCB}_{t,c} := v_{t,c} + 2 C_p \sqrt{\frac{2\log(n_{t,k}}{n_{t,c}}}$
	\EndFor
	\State $D_{t+1} \leftarrow$ $D_{t} \cup D_{\mathrm{new}}$, where $D_{\mathrm{new}}$ is drawn from $\Omega_{k}$ based on qEHVI or CMA-ES.  
    \EndFor

  \end{algorithmic}
\label{alg:lamoo}
\end{algorithm*}

\begin{figure}[ht]
\centering 
\includegraphics[width=1.02\columnwidth]{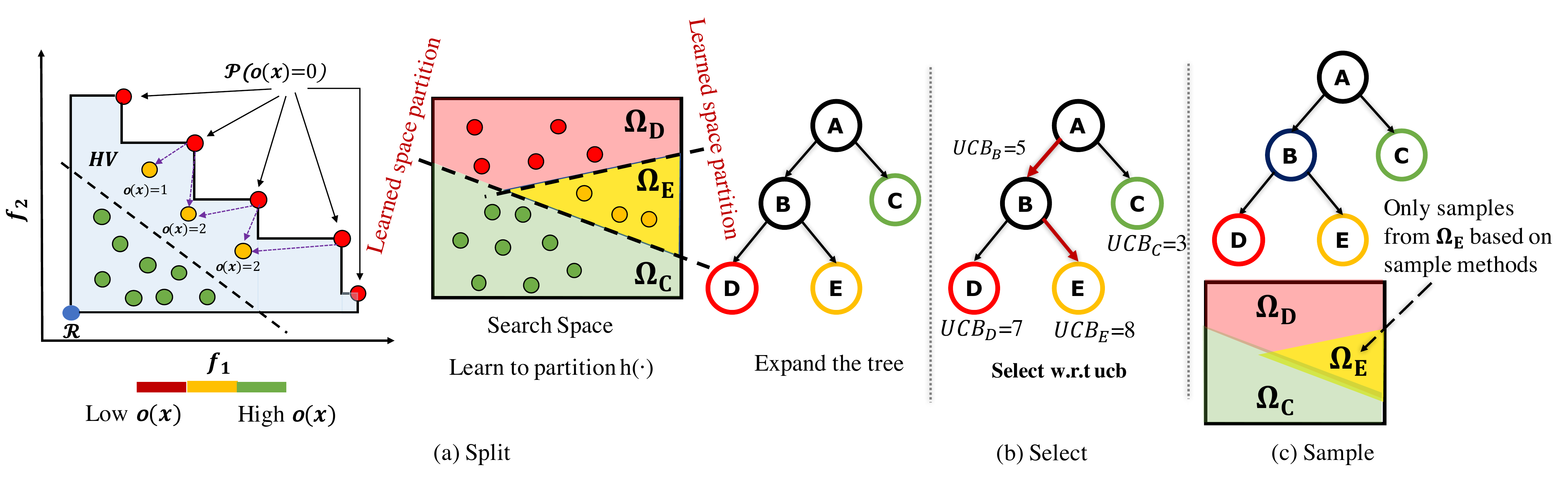}
\vspace{-0.1in}
 \caption{\small (a) The leaf nodes D and E that correspond to the non-splittable space $\Omega_{D}$ and $\Omega_{E}$. (b). The node selection procedure based on the UCB value. (c). The new samples generation from the selected space $\Omega_{E}$ for bayesian optimization.}
\label{fig:workflow}
\end{figure}

\textbf{Learning Space Partitions}. We construct the partition oracle using the \emph{dominance number}. Let $D_t$ be the collected samples up to iteration $t$ and $D_{t,j}:= D_t \cap \Omega_j$ be the samples within the region $\Omega_j$ we want to partition. For each sample $\vx \in D_{t,j}$, its dominance number $o_{t,j}(\vx)$ at iteration $t$ is defined as the number of samples in $\Omega_j$ that dominate $\vx$ (here $\mathbb{I}[\cdot]$ is the indicator function): 
\begin{equation}
o_{t,j}(\vx) := \sum_{\vx_i \in D_{t,j}}\mathbb{I}[\vx \prec_\vf \vx_i,\ \vx \neq \vx_i]
\label{eq:dominance}
\end{equation}
While naive computation requires $O(|D_{t,j}|^2)$ operations, we use Maxima Set~\citep{maximaset} which runs in $O(|D_{t,j}|\log |D_{t,j}|)$. For $\vx \in \Omega_P$, $o(\vx)=0$. 

For each $D_{t,j}$, we then get good (small $o(\vx)$) and bad (large $o(\vx)$) samples by ranking them according to $o(\vx)$. The smallest 50\% are labeled to be positive while others are negative. 

Based on the labeled samples, a classifier (e.g., Support Vector Machine (SVM)) is trained to learn a decision boundary as the latent action. We choose SVM since the classifier needs to be decent in regions with few samples, and has the flexibility of being parametric or non-parametric. 

\label{sec:procedures}
\textbf{Exploration Using UCB}. As shown in Fig.~\ref{fig:workflow}, \ours{} selects the final leaf node by always choosing the child node with larger UCB value. The \emph{UCB value} for a node $j$ is defined as $\mathrm{UCB}_{j} := v_{j} + 2C_{p} \sqrt{2\log n_{\mathrm{parent(j)}}/n_{j}}$, where $n_{j}$ is the number of samples in node $j$, $C_{p}$ is a tunable hyperparameter which controls the degree of exploration, and $v_{j}$ is the hypervolume of the samples in node $j$. The selected leaf corresponds the partitioned region $\Omega_{k}$ as shown in Alg.~\ref{alg:lamoo}.

\textbf{Sampling in Search Region}. 
We use existing algorithms as a sampling strategy in a leaf node, e.g., \qehvi{}~\citep{qehvi}) and CMA-ES~\citep{cma-es}. Therefore, \ours{} can be regarded as a \emph{meta-algorithm}, applicable to any existing SOO/MOO solver to boost its performance.

\underline{\laqehvi{}}. As a multi-objective solver, \qehvi{} finds data points to maximize a parallel version of Expected Hypervolume Improvement (EHVI) via Bayesian Optimization (BO). To incorporate \qehvi{} into LaMOO's sampling step, we confine \qehvi{}'s search space using the tree-structured partition to better search MOO solutions.

\underline{\lacmaes}. \cmaes{} is an evolutionary algorithm (EA) originally designed for single-objective optimization. As a leaf sampler, CMA-ES is used to pick a sample that maximizes the dominance number $o(\vx)$ within the leaf. Since $o(\vx)$ changes over iterations, at iteration $t$, we first update $o_{t'}(\vx)$ of all previous samples at $t'<t$ to $o_t(\vx)$, then use \cmaes{}. Similar to the \qehvi{} case, we constrain our search to be within the leaf region. 

Once a set of new samples $D_{\mathrm{new}}$ is obtained (as well as its multiple function values $\vf(D_{\mathrm{new}})$), we update all partitions along its path and the entire procedure is repeated.

\def\va{\mathbf{a}}
\def\vv{\mathbf{v}}
    
\vspace{-0.1in}
\section{Experiments}
\vspace{-0.1in}

We evaluate the performance of \ours{} in a diverse set of scenarios. This includes synthetic functions, and several real-world MOO problems like neural architecture search, automobile safety design, and molecule discovery. In such real problems, often a bunch of criteria needs to be optimized at the same time. For example, for molecule (drug) discovery, one wants the designed drug to be effective towards the target disease, able to be easily synthesized, and be non-toxic to human body.   
\vspace{-0.1in}
\subsection{Small-scale Problems}
\vspace{-0.1in}
\paragraph{Synthetic Functions.} 
Branin-Currin~\citep{bc_func} is a function with 2-dimensional input and 2 objectives. DTLZ2~\citep{dtlz} is a classical scalable multi-objective problem and is popularly used as a benchmark in the MOO community. We evaluate \ours{} as well as baselines in DTLZ2 with 18 dimensions and 2 objectives, and 12 dimensions and 10 objectives, respectively.

\paragraph{Structural Optimization in Automobile Safety Design (vehicle safety)} is a real-world problem with 5-dimensional input and 3 objectives, including (1) the mass of the vehicle, (2) the collision acceleration in a full-frontal crash, and (3) the toe-board intrusion~\citep{vehicle_safety}. 

\paragraph{Nasbench201} is a public benchmark to evaluate NAS algorithms~\citep{nasbench201}. There are 15625 architectures in Nasbench201, with groundtruth \#FLOPs and accuracy in CIFAR10~\citep{cifar10}. Our goal is to minimize \#FLOPs and maximize accuracy in this search space. We normalized \#FLOPs to range $[-1, 0]$ and accuracy to $[0, 1]$.

\begin{figure}[H]
\centering 
\includegraphics[height=2.5in]{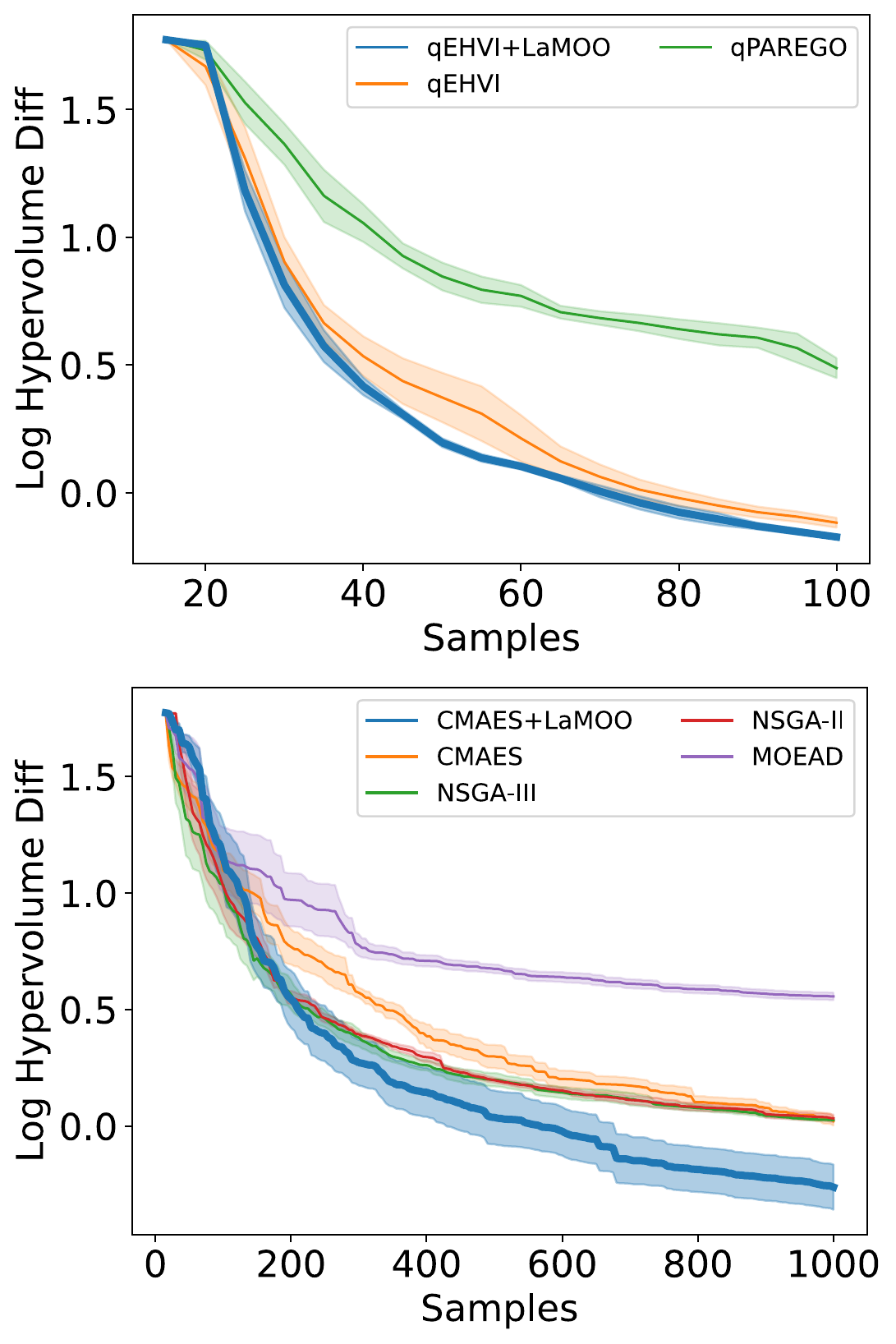}
\includegraphics[height=2.5in]{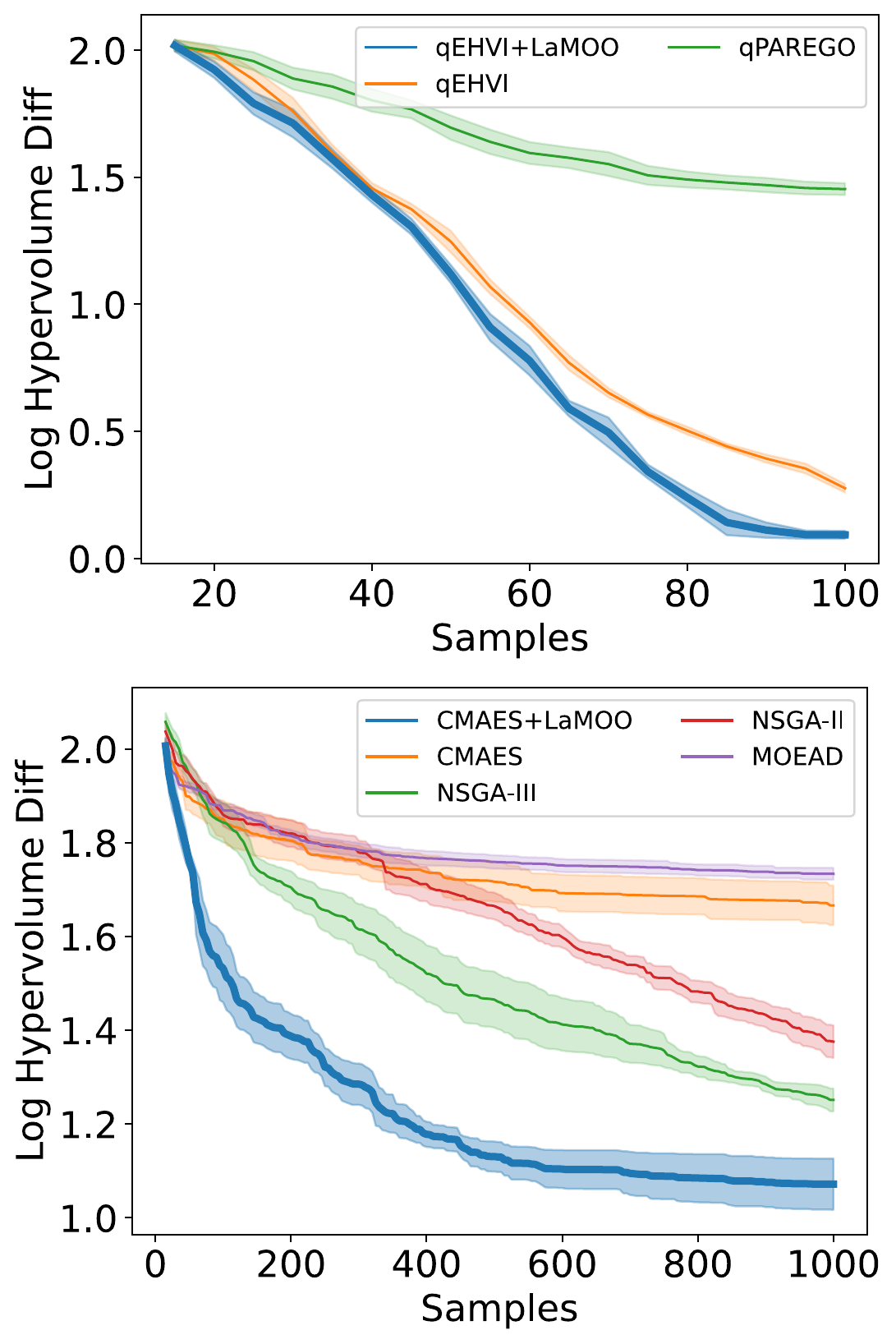}
\includegraphics[height=2.5in]{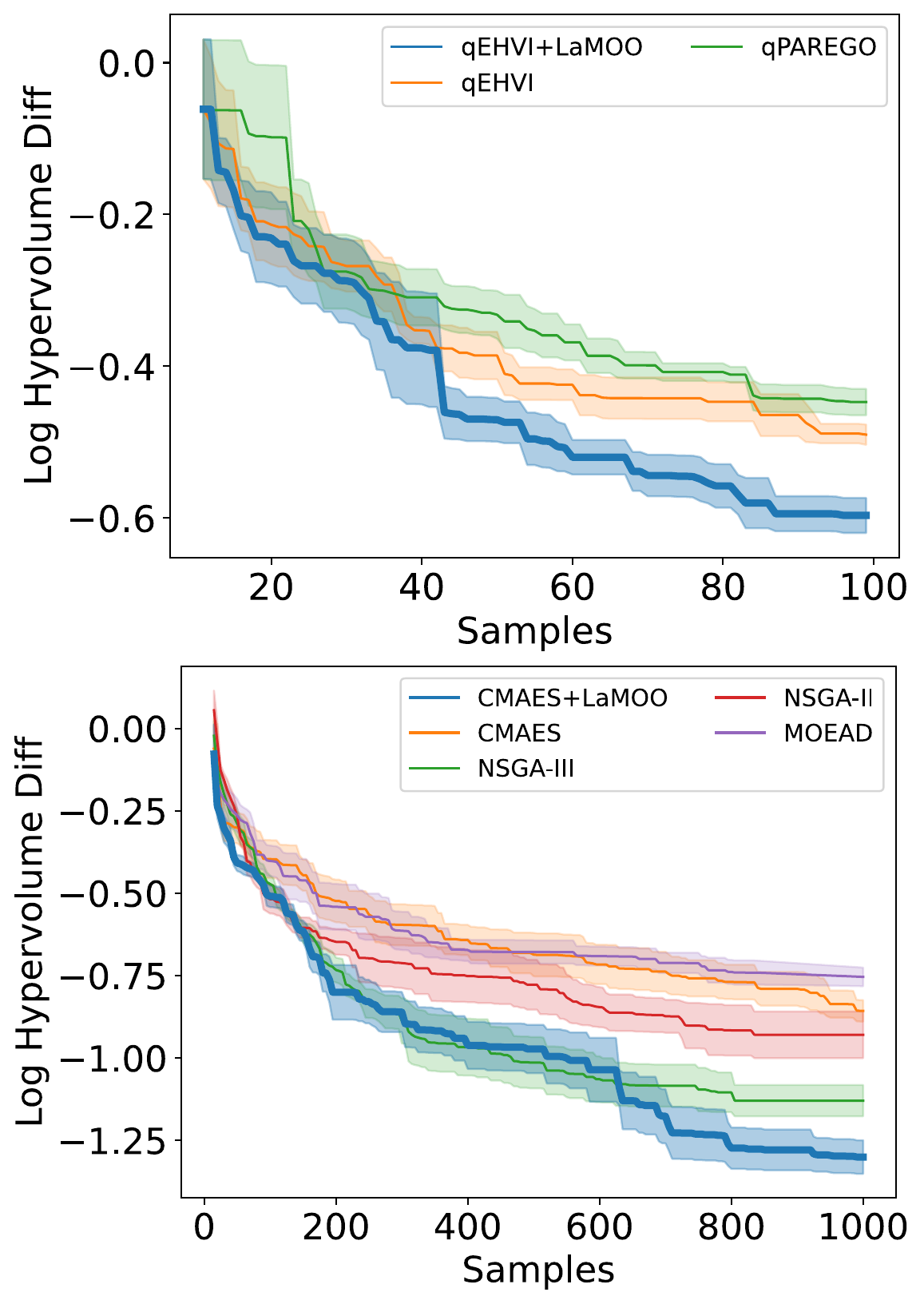}
\vspace{-0.1in}
\caption{\small \textbf{Left}: Branin-Currin with 2 dimensions and 2 objectives. \textbf{Middle}: VehicleSafety with 5 dimensions and 3 objectives. \textbf{Right}: Nasbench201 with 6 dimensions and 2 objectives. We ran each algorithm 7 times (shaded area is $\pm$ std of the mean). \textbf{Top}: Bayesian Optimization w/o \ours{}. \textbf{Bottom}: evolutionary algorithms w/o \ours{}. Note the two algorithm families show very different sample efficiency in MOO tasks.}
\label{fig:small_problems}
\end{figure}

We compare \ours{} with 4 classical evolutionary algorithms (CMA-ES~\citep{cma-es}, MOEA/D~\citep{moead}, NSGA-II~\citep{nsga-ii}, and NSGA-III~\citep{nsgaiii}) and 2 state-of-the-art BO methods (qEHVI~\citep{qehvi} and qParego~\citep{qparego}). 

\textbf{Evaluation Criterion}. we first obtain the maximal hypervolume (either by ground truth or from the estimation of massive sampling), then run each algorithm and compute the log hypervolume difference~\citep{qehvi}:
\begin{equation}
HV_{\mathrm{log\_diff}} := \log(HV_{\mathrm{max}} - HV_{\mathrm{cur}})
\end{equation}
where $HV_{\mathrm{cur}}$ is the hypervolume of current samples obtained by the algorithm with given budget.

\textbf{Result.} As shown in Fig.~\ref{fig:small_problems}, \ours{} with qEHVI outperforms all our BO baselines and \ours{} with CMA-ES outperforms all our EA baselines, in terms of $HV_{\mathrm{log\_diff}}$.

Evolutionary algorithms rely on mutation and crossover of previous samples to generate new ones, and may be trapped into local optima. Thanks to MCTS, LaMOO also considers exploration and greatly improves upon vanilla CMA-ES over three different tasks with 1000/200 samples in small-scale/many objective problems.. In addition, by plugging in BO, \ours{}+qEHVI achieve 225\% sample efficiency compared to other BO algorithms on Nasbench201. This result indicates that for high-dimensional problems (6 in Nasbench201), space partitioning leads to faster optimization. We further analyze very high-dimensional problems on Sec.~\ref{sec:Modules}. For visualization of Pareto frontier by LaMOO+qEHVI, see Fig.~\ref{fig:pareto} in Appendix.

\begin{figure}[t]
\centering 
\includegraphics[height=0.98in]{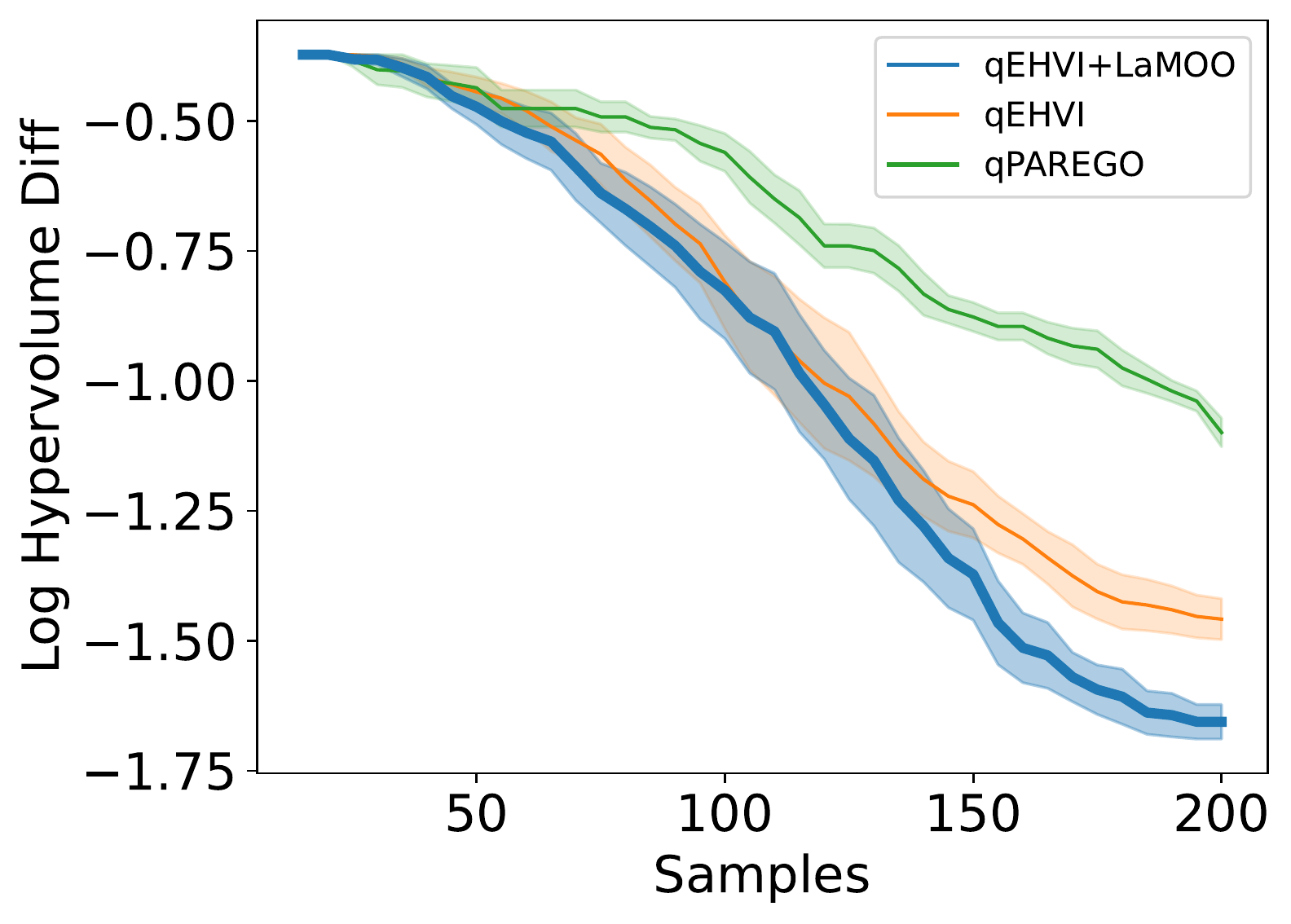}
\includegraphics[height=0.98in]{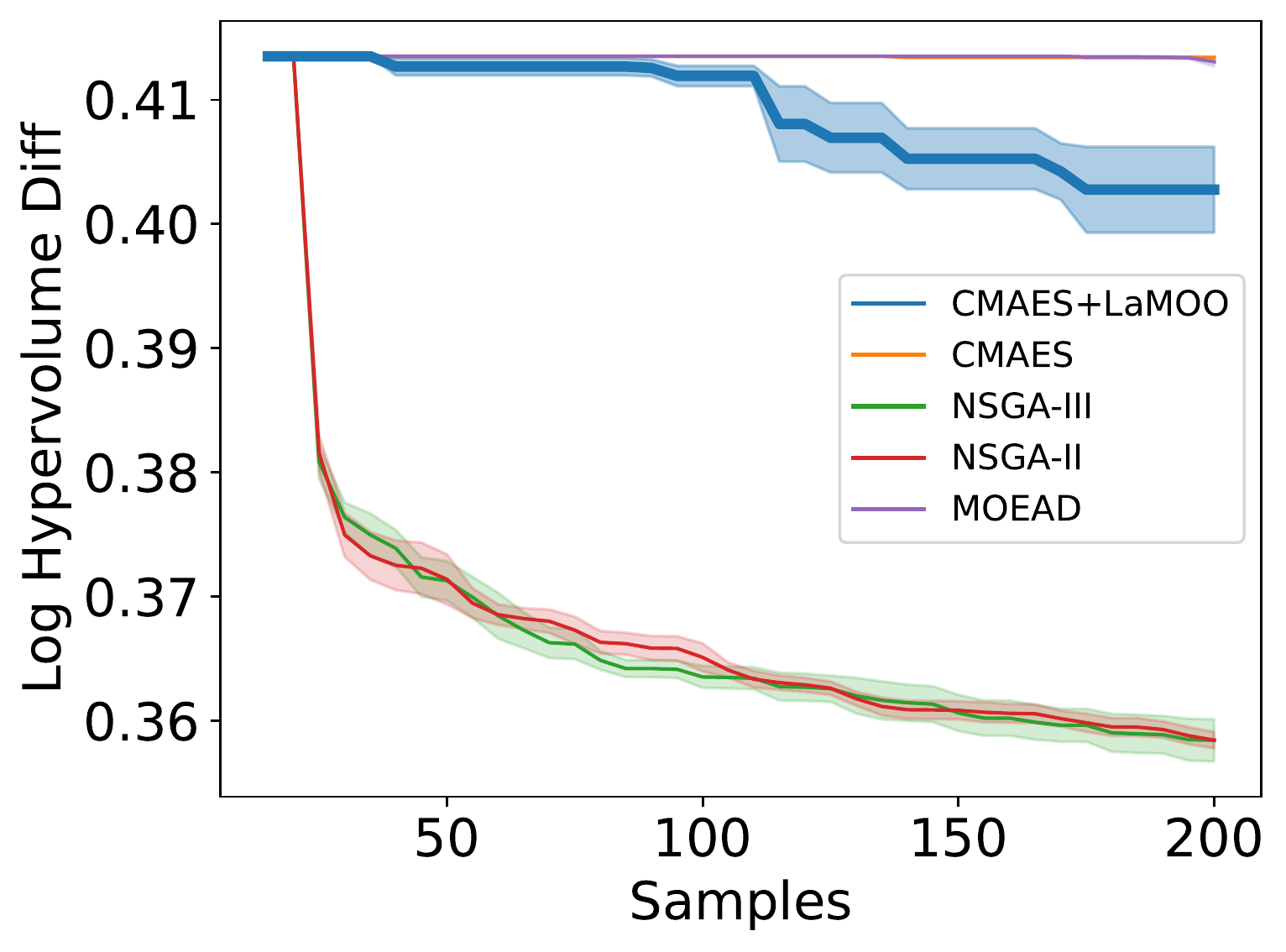}
\includegraphics[height=0.98in]{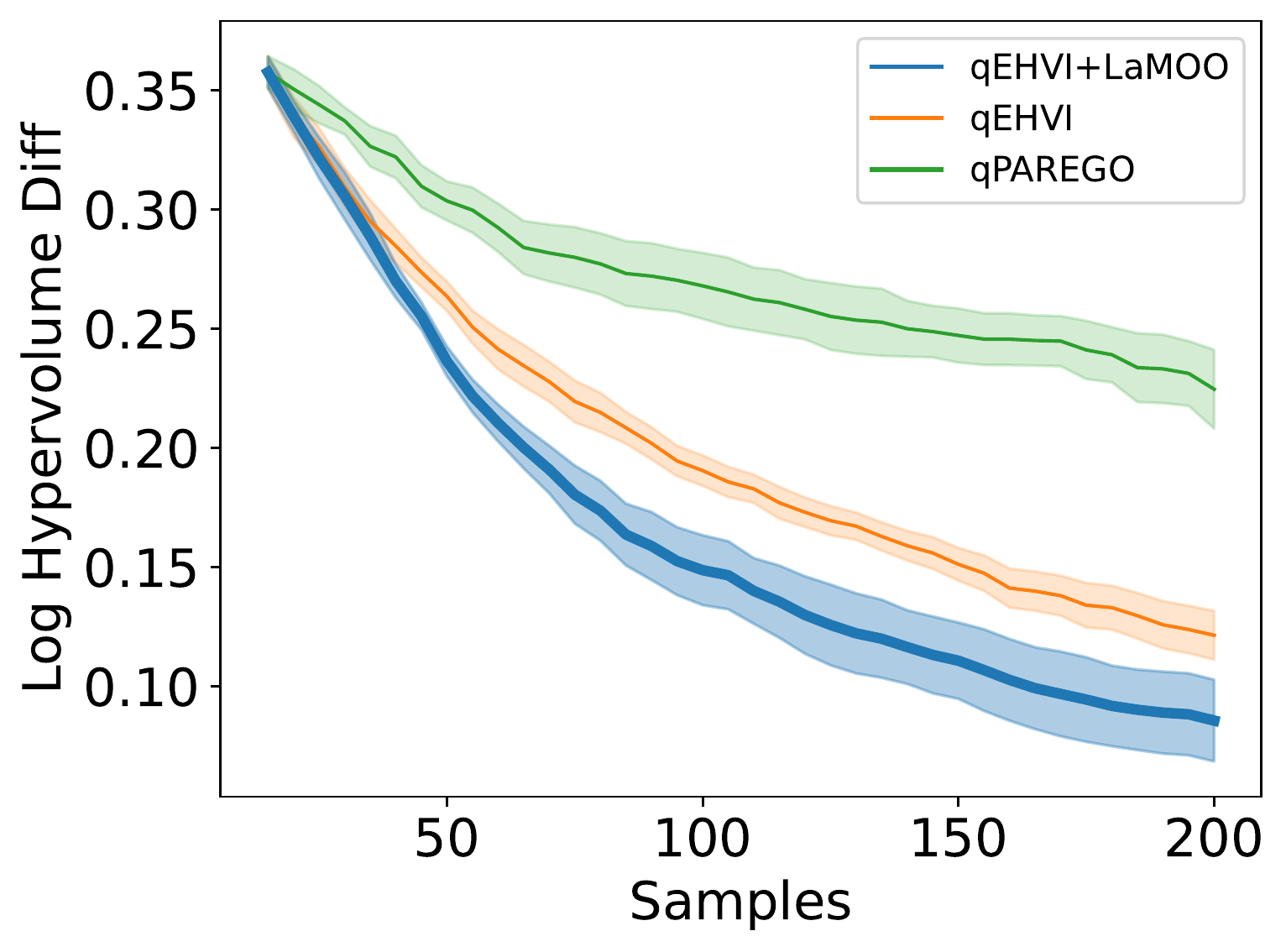}
\includegraphics[height=0.98in]{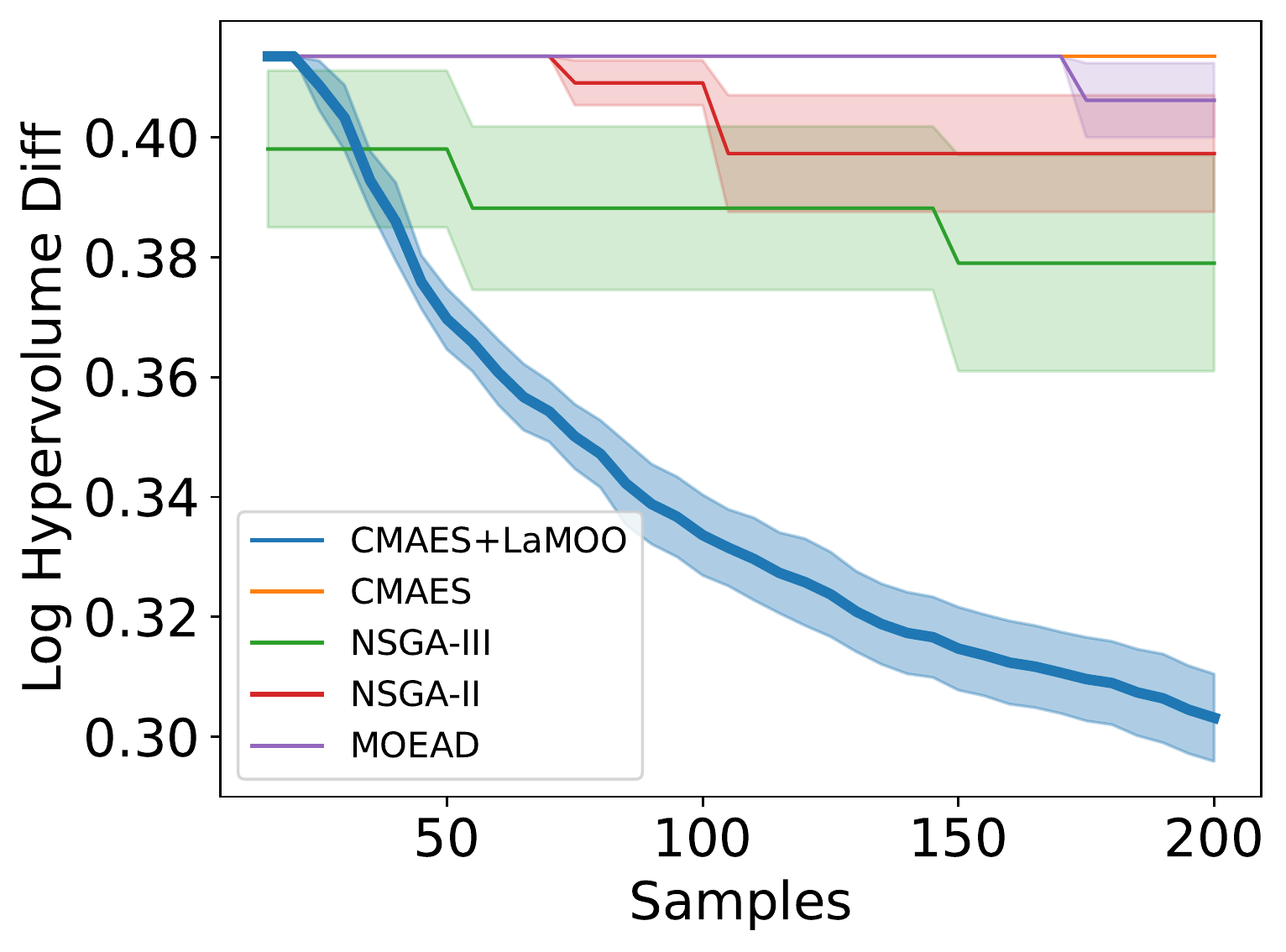}
\vspace{-0.1in}
\caption{\small DTLZ2 with many objectives, We ran each algorithm 7 times (shaded area is $\pm$ std of the mean). From \textbf{left} to \textbf{right}: BO with 2 objectives; EA with 2 objectives; BO with 10 objectives; EA with 10 objectives. } 
\label{fig:DTLZ2}
\end{figure}

\textbf{Optimization of Many Objectives}.
While NSGA-II and NSGA-III perform well in the two-objective problems, all evolutionary-based baselines get stuck in the ten-objective problems. In contrast, LaMOO performs reasonably well. From Fig.~\ref{fig:DTLZ2}, qEHVI+LaMOO shows strong performance in ten objectives. When combined with a CMA-ES, \ours{} help it escape the initial region to focus on a smaller promising region by space partitioning.

\vspace{-0.1in}
\subsection{Multi-Objective Molecule Discovery}
\vspace{-0.1in}
\label{sec:Modules}

\begin{figure}[H]
\centering 
\includegraphics[height=1.33in]{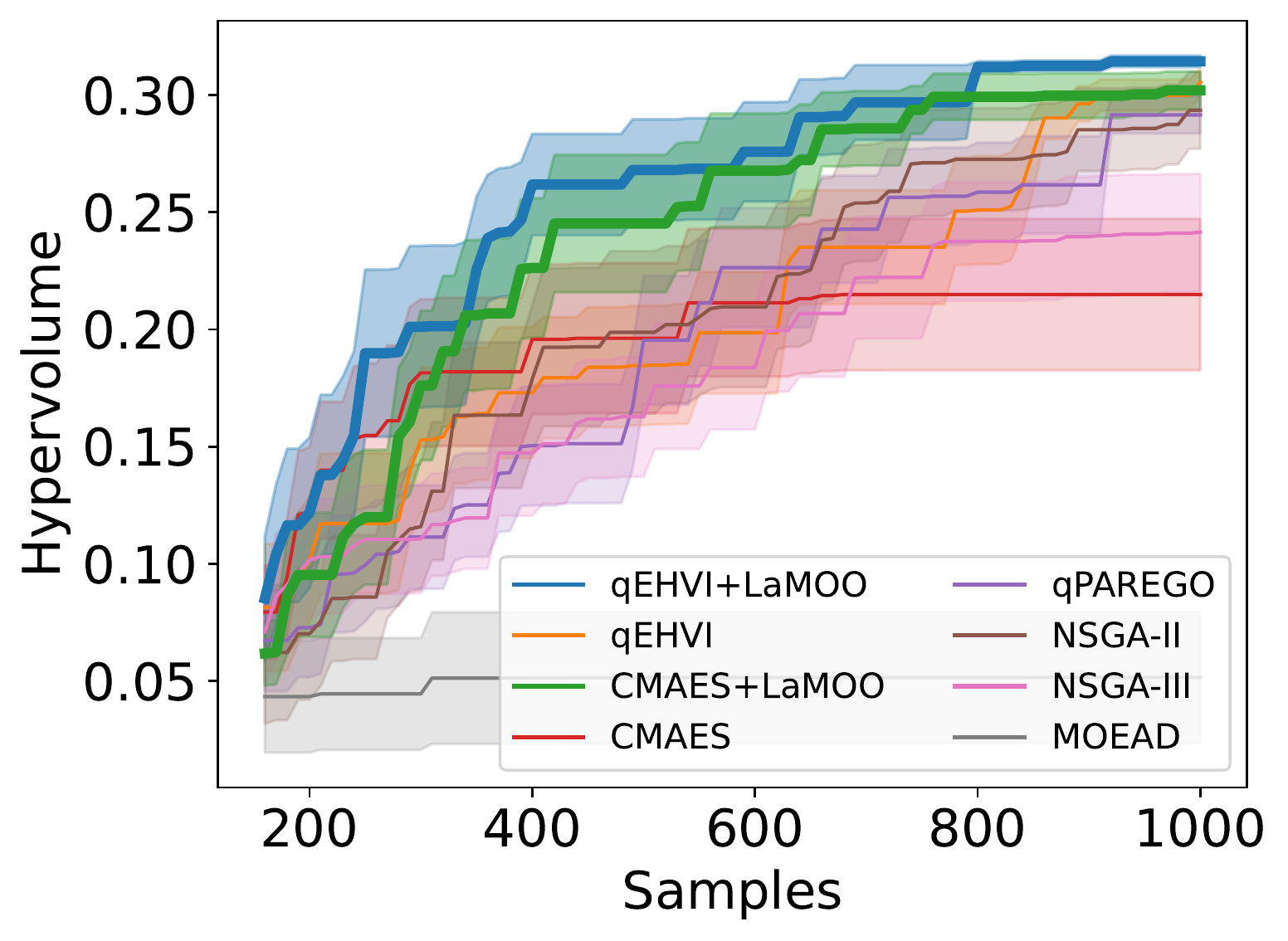}
\includegraphics[height=1.33in]{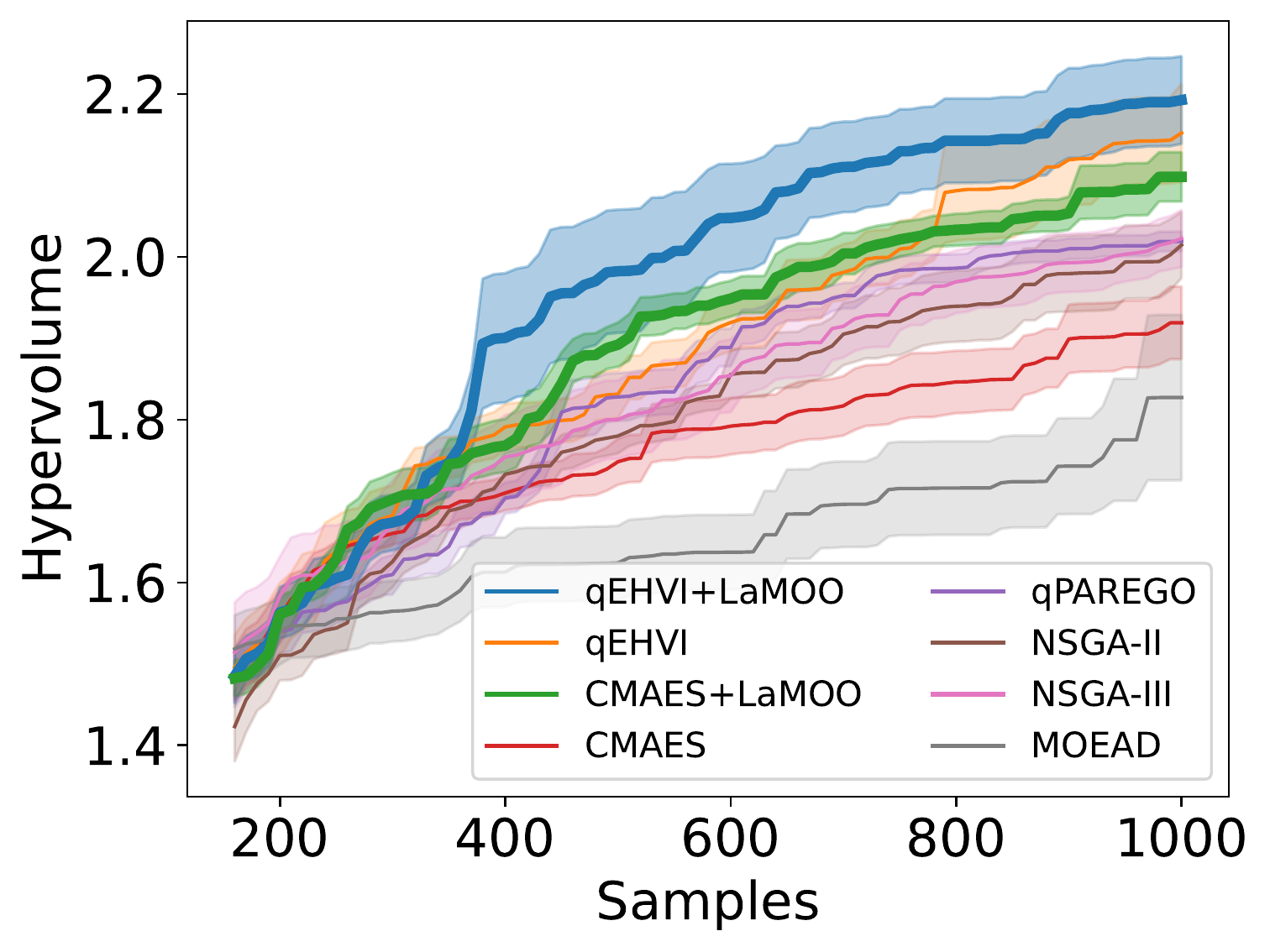}
\includegraphics[height=1.33in]{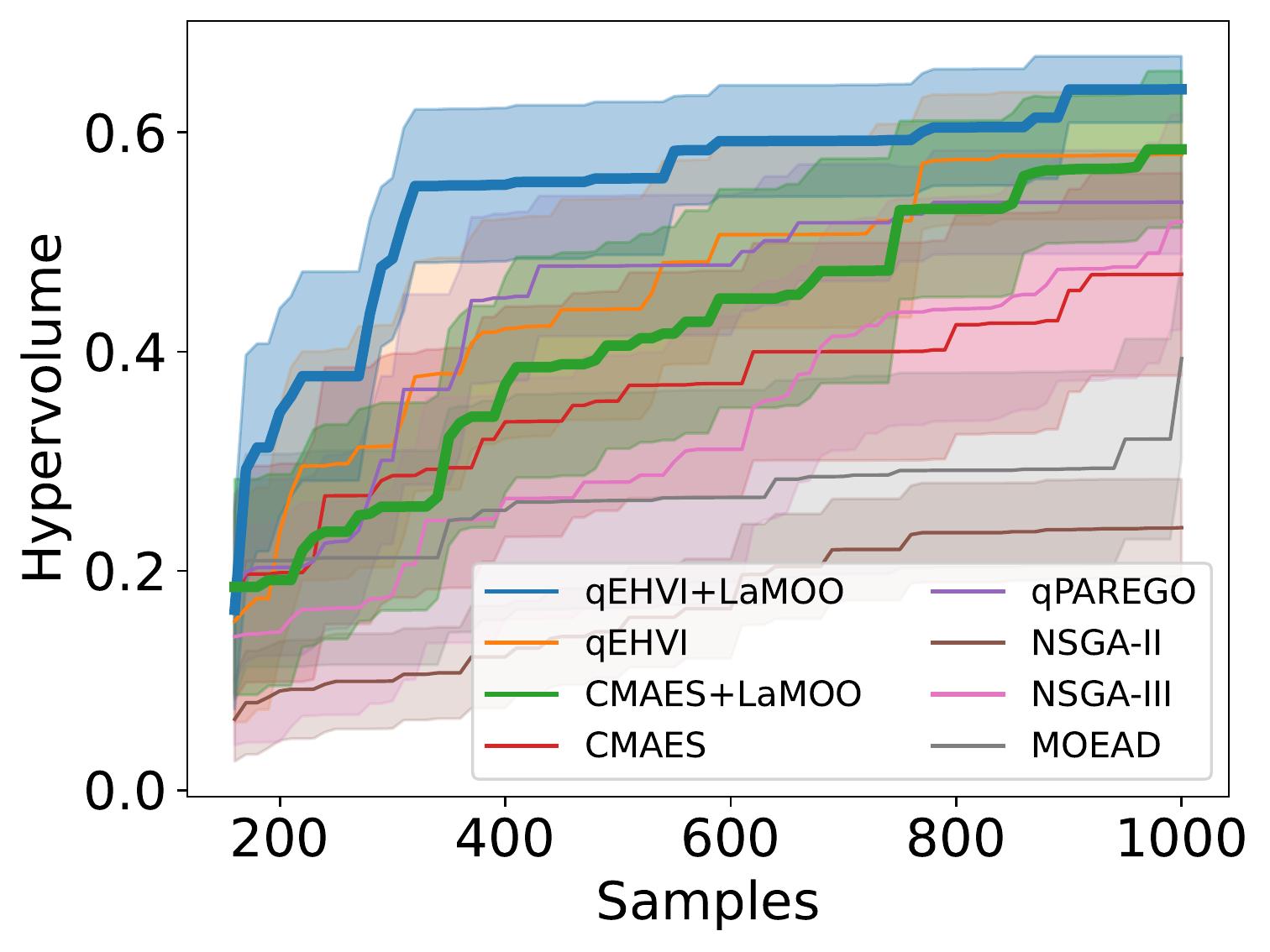}
\vspace{-0.1in}
\caption{\small Molecule Discovery: \textbf{Left}: Molecule discovery with two objectives (GSK3$\beta$+JNK3). \textbf{Middle}: Molecule discovery with three objectives (QED+SA+SARS). \textbf{Right}: Molecule Discovery with four objectives (GSK3$\beta$+JNK3+QED+SA). We ran each algorithm 15 times (shaded area is $\pm$ std of the mean).}
\label{fig:molecule}
\end{figure}

Next, we tackle the practical problem of multi-objective molecular generation, which is a high-dimensional problem (search space is 32-dimensional). Molecular generation models are a critical component of pharmaceutical drug discovery, wherein a cheap-to-run \textit{in silico} model proposes promising molecules which can then be synthesized and tested in a lab~\citep{vamathevan2019applications}. However, one commonly requires the generated molecule to satisfy multiple constraints: for example, new drugs should generally be non-toxic and ideally easy-to-synthesize, in addition to their primary purpose. Therefore, in this work, we consider several multi-objective molecular generation setups from prior work on molecular generation~\citep{yu2019evaluating,jin2020multi,plalam}: (1) activity against biological targets GSK3$\beta$ and JNK3, (2) the same targets together with QED (a standard measure of ``drug-likeness'') and SA (a standard measure of synthetic accessibility), and (3) activity against SARS together with QED and SA. In each task, we propose samples from a pre-trained 32-dimensional latent space from \citep{jin2020hierarchical}, which are then decoded into molecular strings and fed into the property evaluators from prior work.

Fig.~\ref{fig:molecule} shows that LaMOO+qEHVI outperforms all baselines by up to 10\% on various combinations of objectives. While EA struggles to optimize these high-dimensional problems due to the limitations mentioned in Sec.~\ref{related-works}, \ours{} help them (e.g., CMA-ES) to performs much better.

\vspace{-0.1in}
\subsection{Ablation studies}
\vspace{-0.1in}

\begin{figure}[t]
\centering 
\includegraphics[width=1.02\columnwidth]{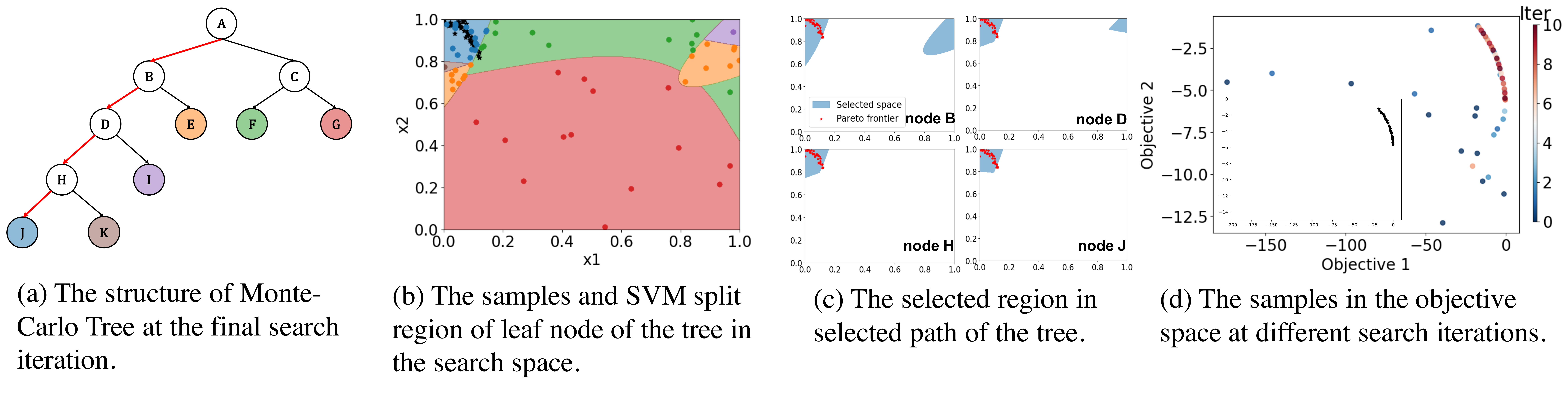}
\vspace{-0.2in}
 \caption{\small Visualization of selected region at different search iterations and nodes. (a) The Monte-Carlo tree with colored leaves. Selected path is marked in red. (b) Visualization of the regions($\Omega_{J}, \Omega_{K} , \Omega_{I}, \Omega_{E}, \Omega_{F}, \Omega_{G}$) that are consistent with leaves in (a) in the search space.  (c) Visualization of selected path at final iteration.  (d) Visualization of samples during search; bottom left is the Pareto frontier estimated from one million samples.}
\label{fig:verf}
\end{figure}

\textbf{Visualization of \ours{}.} To understand how \ours{} works, we visualize its optimization procedure for Branin-Currin. First, the Pareto optimal set $\Omega_P$ is estimated from $10^6$ random samples (marked as black stars), as shown in both search and objective space (Fig.~\ref{fig:verf}(b) and bottom left of Fig.~\ref{fig:verf}(c)). Over several iterations, \ours{} progressively prunes away unpromising regions so that the remaining regions approach $\Omega_P$ (Fig~\ref{fig:verf}(c)). Fig~\ref{fig:verf}(a) shows the final tree structure. The color of each leaf node corresponds to a region in the search space (Fig~\ref{fig:verf}(b)). The selected region is recursively bounded by SVM classifiers corresponding to nodes on the selected path (red arrow in Fig~\ref{fig:verf}(a)). The new samples are only generated from the most promising region $\Omega_{J}$, improving sample efficiency.

\label{sec:ablation_hyper_params}

\begin{figure}[H]
\centering 
\includegraphics[height=1.33in]{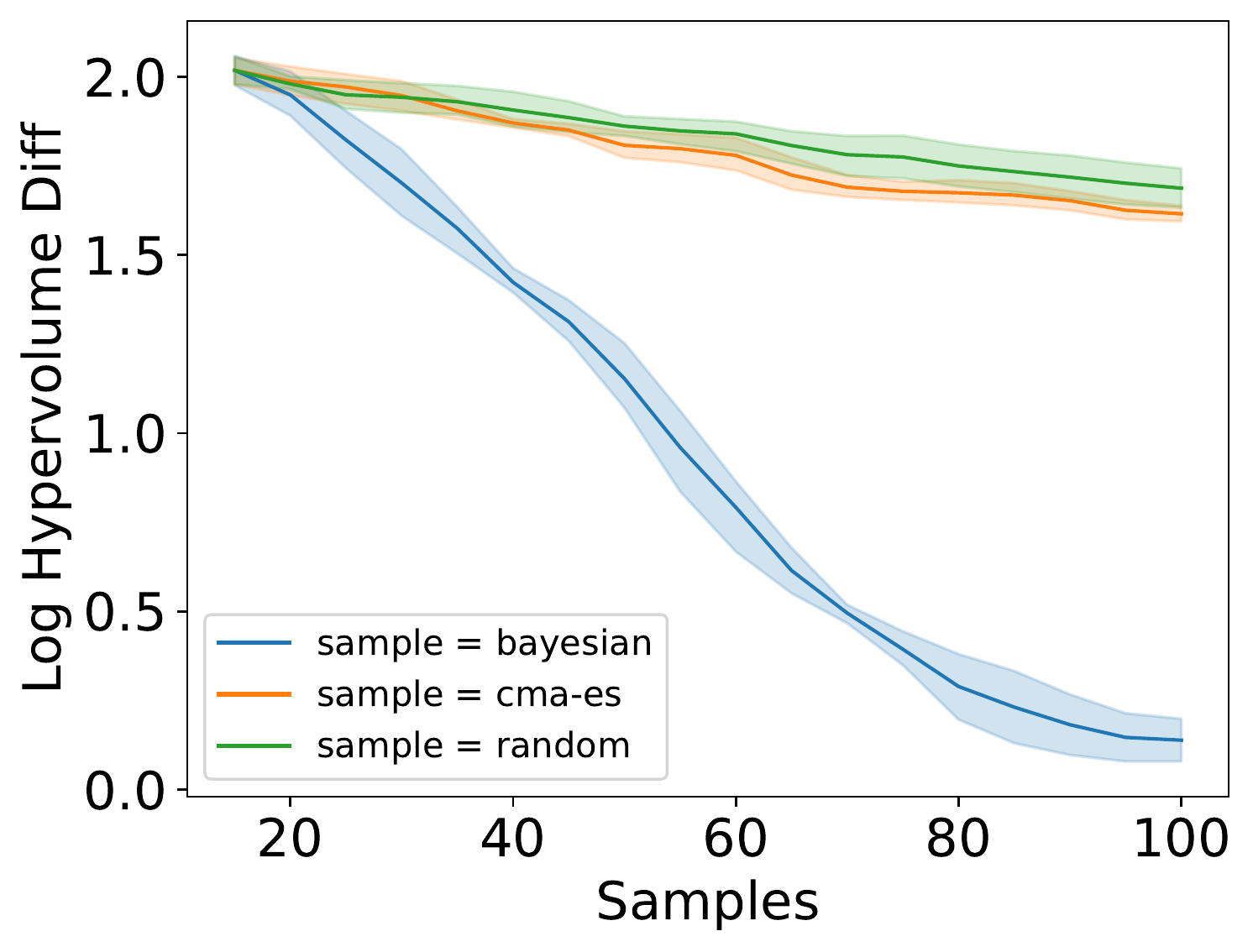}
\includegraphics[height=1.33in]{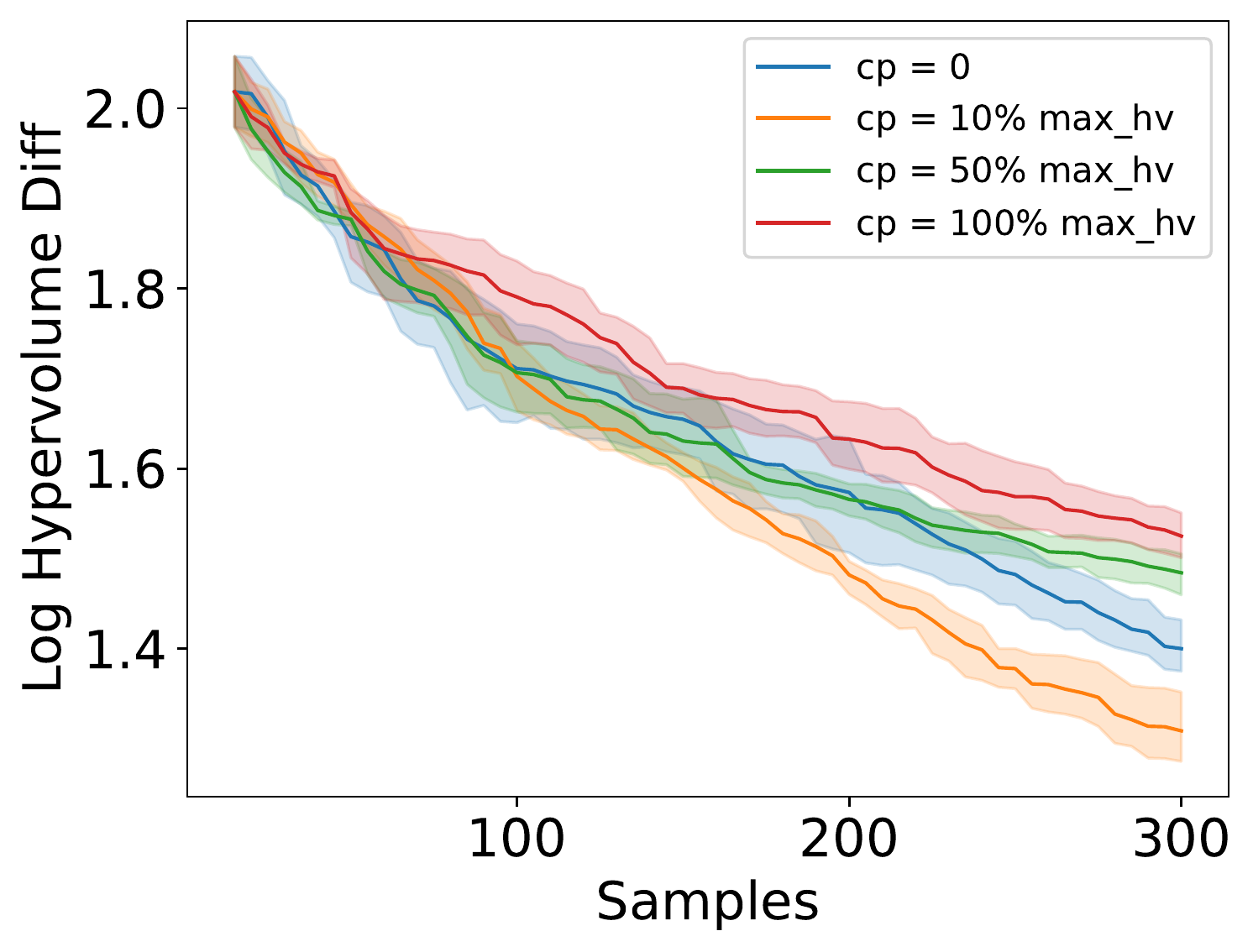}
\includegraphics[height=1.33in]{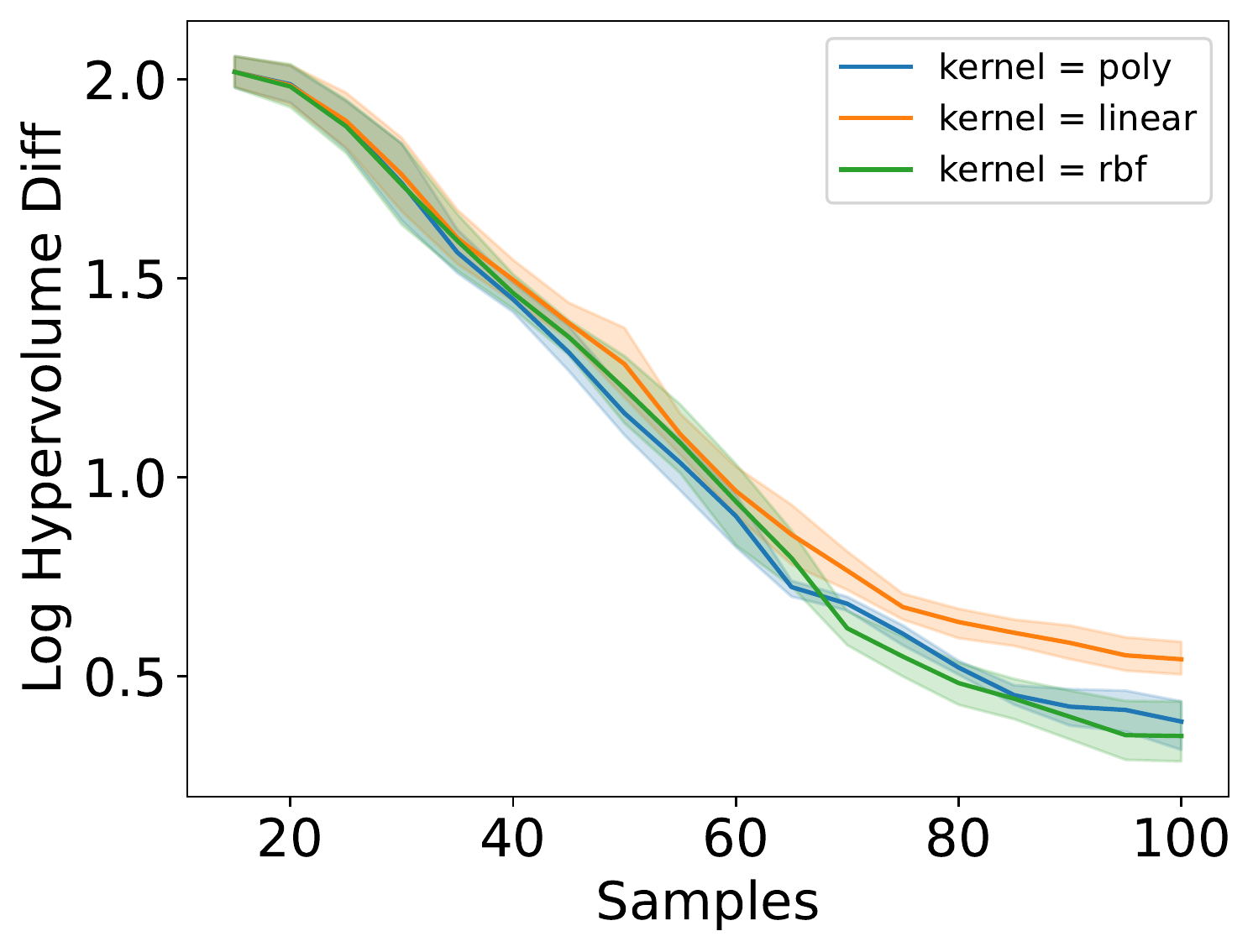}
\vspace{-0.1in} 
\caption{\small Ablation studies on hyperparameters and sampling methods in \ours{}. \textbf{Left}: Sampling without Bayesian/CMA-ES. \textbf{Middle}: Sampling with different $C_p$. \textbf{Right}: Partitioning with different svm kernels}
\label{fig:ablation}
\end{figure}

\vspace{-0.1in}

\textbf{Ablation of Design Choices.} We show how different hyperparameters and sampling methods play a role in the performance. We perform the study in VehicleSafety below.  

\emph{Sampling methods}. LaMOO can be integrated with different sample methods, including Bayesian Optimization (e.g., qEHVI) and evolutionary algorithms (e.g., CMA-ES). Fig.~\ref{fig:ablation}(left) shows that compared to random sampling, qEHVI improves a lot while CMA-ES only improves slightly. This is consistent with our previous finding that for MOO, BO is much more efficient than EA. 

\emph{The exploration factor $C_p$} controls the balance of exploration and exploitation. A larger $C_p$ guides LaMOO to visit the sub-optimal regions more often. Based on the results in Fig.~\ref{fig:ablation}(middle), greedy search ($C_p=0$) leads to worse performance compared to a proper $C_p$ value (i.e. 10\% of maximum hypervolume), which justifies our usage of MCTS. On the other hand, over-exploration can also yield even worse results than greedy search. Therefore, a "rule of thumb" is to set the $C_p$ to be roughly 10\% of the maximum hypervolume $HV_{\max}$. When $HV_{\max}$ is unknown, $C_p$ can be set empirically.

\emph{SVM kernels}. As shown in Fig.~\ref{fig:ablation}(right), we find that the RBF kernel performs the best, in agreement with ~\citep{wang2020learning}. Thanks to the non-linearity of the polynomial and RBF kernels, their region partitions perform better compared to a linear one.

\section{Conclusion}
We propose a search space partition optimizer called \ours{} as a meta-algorithm that extends prior single-objective works~\citep{wang2020learning, plalam} to multi-objective optimization. We demonstrated both theoretically and via experiments on multiple MOO tasks that \ours{} significantly improves the search performance compared to strong baselines like qEHVI and CMA-ES. In particular, \ours{} + qEHVI achieves state-of-the-art results on a variety of tasks including synthetic functions, neural architecture search, and molecule discovery. 

\section{Acknowledgement}
This work was supported in part by NSF Grants \#1815619 and \#2105564, a VMWare grant, and computational resources supported by the Academic \& Research Computing group at Worcester Polytechnic Institute.

\bibliography{output}
\bibliographystyle{iclr2022_conference}

\newpage

\setcounter{lemma}{0}
\setcounter{theorem}{0}
\setcounter{observation}{0}

\appendix
\section{Proofs}\label{sec:proofs}

\begin{lemma}
The algorithm to uniformly draw $k$ samples in $S$, pick the best and return is a $(1,1)$-oracle.  
\end{lemma}
\begin{proof}
Consider the following simple $(1, 1)$-oracle for single-objective optimization: after sampling $k$ samples, we rank them according to their function values, and split them into two $k/2$ smaller subsets $\tilde S_\good$ and $\tilde S_\bad$. Other points are randomly assigned to either of the two subsets. Then if $\vx^*$ happens to be among the $k$ collected samples (which happens with probability $k / |S|$), definitely we have $\vx^*\in S_\good$. Therefore, we have:
\begin{equation}
    P\left(\vx^* \in S_\good | \vx^* \in S\right) \ge \frac{k}{|S|} \ge 1 - \exp\left(- \frac{k}{|S|}\right)
\end{equation}
which is an oracle with $\alpha = \eta = 1$. The last inequality is due to $e^x \ge 1 + x$ (and thus $e^{-x} \ge 1 - x$). 
\end{proof}

\begin{lemma}
\label{lemma:optimal-action}
Define $g(\lambda): \rr^+ \mapsto \rr^+$ as: 
\begin{equation}
g(\lambda): \lambda \mapsto \sum^{T}_{t=1} w_t \log \left( 1 + \dfrac{1}{\lambda w_t}\right)
\end{equation}
The following maximization problem 
\begin{equation}
\max _{\{ z_{t}\} }\sum _{t=1}^{T}\log \left( 1-e^{-z_t}\right)\quad \mathrm{s.t.} \ \sum _{t=1}^T w_{t}z_{t}=K
\end{equation}
has optimal solutions 
\begin{equation}
z^*_{t}=\log\left( 1+\dfrac{1}{\lambda w_t}\right),\quad 1\le t\le T 
\end{equation}
where $\lambda$ is determined by $g(\lambda) = K$. With optimal $\{z^*_t\}$, the objective reaches $-\sum_t \log (1+\lambda w_t)$. 
\end{lemma}
\begin{proof}
Its Lagrange is:
\begin{equation}
J\left(\{z_{t}\}\right) =\sum_{t=1}^{T}\log\left( 1-e^{-z_t}\right) -\lambda \left( \sum_{t=1}^T w_{t}z_{t}-K\right)
\end{equation}

Taking derivative w.r.t. $z_t$ and we have:
\begin{equation}
\begin{aligned}\dfrac{\partial J }{\partial z_{t}}=\dfrac{e^{-z_t}}{1-e^{-z_t}}-\lambda w_{t}=0.\\
\dfrac{1}{1-e^{-z_t}}-1-\lambda w_{t}=0\\
\dfrac{1}{1-e^{-z_t}}=1+\lambda w_{t}\\
1-e^{-z_{t}}=\dfrac{1}{1+\lambda w_t}\\
e^{-z_{t}}=1-\dfrac{1}{1+\lambda w_{t}}=\dfrac{\lambda w_{t}}{1+\lambda w _{t}}\\
z_{t}=-\log \dfrac{\lambda w_{t}}{1+\lambda w_{t}}=\log \dfrac{1+\lambda w _{t}}{\lambda w_{t}} =\log\left( 1+\dfrac{1}{\lambda w_t}\right) \end{aligned}
\end{equation}
\end{proof}

\begin{lemma}
\label{lemma:g-property}
Both $g(\lambda)$ and $g^{-1}(y)$ are monotonously decreasing. Furthermore, let $\bar w := T \left(\sum_{t=1}^T w_t^{-1}\right)^{-1}$ be the Harmonic mean of $\{w_t\}$ and $w_{\max} := \max_{t=1}^T w_t$, we have:
\begin{equation}
    \frac{\bar w^{-1}}{\exp(\bar w^{-1} y / T) - 1} \le g^{-1}(y) \le \frac{w^{-1}_{\max}}{\exp(w^{-1}_{\max} y / T) - 1} \le \frac{T}{y}.
\end{equation}
\end{lemma}
\begin{figure}
    \centering
    \includegraphics[width=\textwidth]{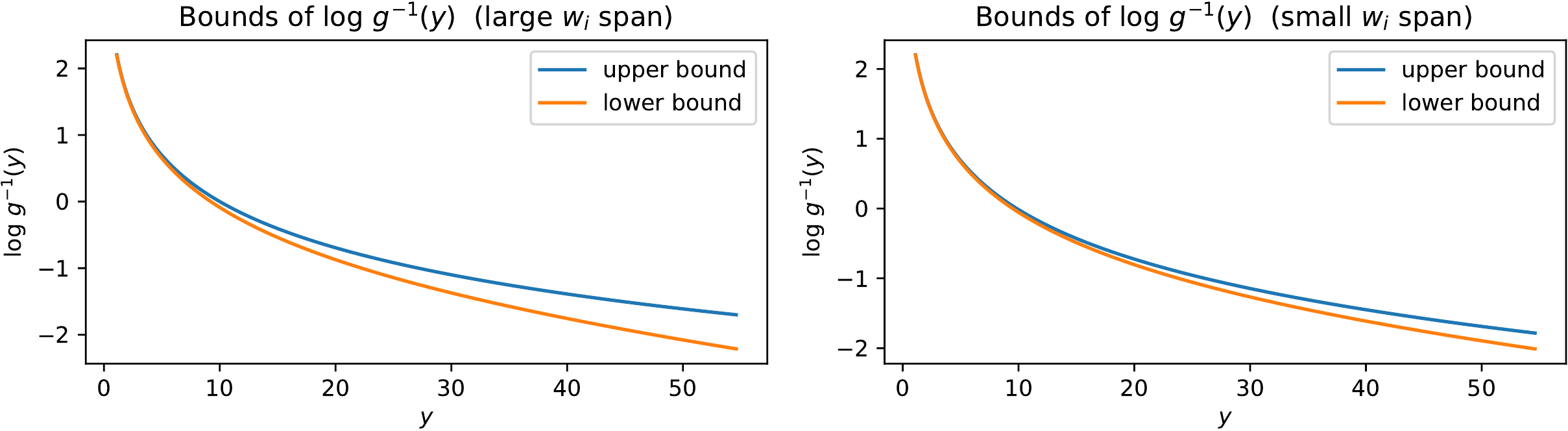}
    \caption{Upper and lower bounds of $g^{-1}(y)$ with different $\{w_i\}$. Left: $w_i = 2^{\mathrm{linspace}(-0.1, 10)}$. Right: $w_i = 2^{\mathrm{linspace}(2, 5)}$. Small $\{w_i\}$ span leads to better bounds.}
    \label{fig:bounds}
\end{figure}

\begin{proof}
It is easy to see when $\lambda$ increases, each term in $g(\lambda)$ decreases and thus $g(\lambda)$ is a decreasing function of $\lambda$. Therefore, its inverse mapping $g^{-1}(y)$ is also decreasing. 

Let $\mu(y) := 1 / g^{-1}(y) > 0$. Then we have:
\begin{equation}
y = \sum_{t=1}^T w_t \log\left(1 + \frac{\mu(y)}{w_t}\right)  
\end{equation}
It is clear that when $y = 0$, $\mu(y) = 0$. Taking derivative with respect to $y$ in both side, we have:
\begin{equation}
1 = \mu'(y) \sum_{t=1}^T \frac{1}{1 + \frac{\mu(y)}{w_t}}  
\end{equation}
where $\mu'(y) = \frac{\dd \mu(y)}{\dd y}$ is the derivative of $\mu(y)$. Using the property of Harmonic mean, we have:
\begin{equation}
    \mu'(y) = \left(\sum_{t=1}^T \frac{1}{1 + \frac{\mu(y)}{w_t}}\right)^{-1} \le \frac{\sum_{t=1}^T 1 + \frac{\mu(y)}{w_t}}{T^2} = \frac{1}{T}\left(1 + \frac{\mu(y)}{\bar w}\right)
\end{equation}
This gives:
\begin{equation}
    \frac{\mu'(y)}{1 + \mu(y) / \bar w} \le \frac{1}{T}
\end{equation}
Integrate on both side starting from $y = 0$, we have:
\begin{equation}
    \bar w \log(1 + \mu(y) / \bar w)\Bigg|_0^y \le \frac{y}{T}\Bigg|_0^y 
\end{equation}
Using $\mu(0) = 0$ we thus have:
\begin{equation}
    \bar w \log(1 + \mu(y) / \bar w) \le \frac{y}{T}
\end{equation}
This leads to $\mu(y) \le \bar w \left[\exp(y\bar w^{-1} / T) - 1\right]$. With $g^{-1}(y) = 1 / \mu(y)$, we arrive the final lower bound for $g^{-1}(y)$.  

For an alternative upper bound of $g^{-1}(y)$, we just notice that (here $w_{\max} := \max_t w_t$):
\begin{equation}
    \mu'(y) = \left(\sum_{t=1}^T \frac{1}{1 + \frac{\mu(y)}{w_t}}\right)^{-1} \ge \left(\frac{T}{1 + \frac{\mu(y)}{w_{\max}}}\right)^{-1} = \frac{1}{T}\left(1 + \frac{\mu(y)}{w_{\max}}\right)
\end{equation}
Using the same technique as above, we have $\mu(y) \ge w_{\max} \left[\exp(y w^{-1}_{\max} / T) - 1\right]$ and the upper bound of $g^{-1}(y)$ follows. 

Finally, note that $e^x \ge 1 + x$, we have
\begin{equation}
 \frac{w^{-1}_{\max}}{\exp(w^{-1}_{\max} y / T) - 1} \le \frac{w^{-1}_{\max}}{w^{-1}_{\max} y / T} =    \frac{T}{y}    
\end{equation}

\end{proof}
\begin{theorem}
\label{thm:reward}
Following optimal sequence, the algorithm yields a reward $r^*$ lower bounded by the following:
\begin{equation}
    r^* \ge r_\baseline \exp\left[\left(\log 2 - \frac{\eta N^\alpha \phi(\alpha, T)}{K}\right)T\right] 
\end{equation}
where $r_\baseline := N^{-1}$ and $\phi(\alpha, T) := (1 - 2^{-\alpha T}) / (1 - 2^{-\alpha})$.
\end{theorem}
\begin{proof}
First note that $|S_T| \le |S_0| / 2^T$ and thus $\frac{1}{|S_T|} \ge 2^T / N$. So we just need to bound $P(\vx^* \in S_T)$, which can be written as:
\begin{eqnarray}
    P(\vx^* \in S_T) = \prod_{t=1}^T P(\vx^* \in S_t | \vx^* \in S_{t-1}) \ge \prod_{t=1}^T \left(1 - \exp\left(- \frac{k_t}{\eta |S_{t-1}|^\alpha}\right)\right)
\end{eqnarray}
Therefore we have
\begin{equation}
\log P(\vx^* \in S_T) \ge \sum_{t=1}^T \log\left(1 - \exp\left(- \frac{k_t}{\eta |S_{t-1}|^\alpha}\right)\right) 
\end{equation}
We want to find the action sequence $\{k_t\}$ so that $\log P(\vx^*\in S_T)$ is maximized. Let $w_t := \eta|S_{t-1}|^\alpha$ and $z_t := k_t / w_t$, applying Lemma~\ref{lemma:optimal-action}, and we know that 
\begin{equation}
    \max_{\{k_t\}} \log P(\vx^* \in S_T) \ge - \sum_{t=1}^T\log (1+ \lambda w_t) 
\end{equation}
where the Lagrangian multiplier $\lambda$ satisfies the equation $g(\lambda) = K$.

Now we have:
\begin{eqnarray}
    \sum_{t=1}^T\log (1+ \lambda w_t) &\stackrel{\circle{1}}{\le}& \sum_{t=1}^T\log \left(1+ \frac{T}{K} w_t\right) \\ 
    &\stackrel{\circle{2}}{\le}& \sum_{t=1}^T\log \left(1+ \frac{T}{K} \eta (N/ 2^{t-1})^\alpha\right) \\ 
    &\stackrel{\circle{3}}{\le}& \frac{\eta T N^\alpha}{K} \sum_{t=1}^T \frac{1}{2^{\alpha(t-1)}} \\
    &= & \phi(\alpha, T)\frac{\eta T N^\alpha}{K}  
\end{eqnarray}
Here $\circle{1}$ is due to Lemma~\ref{lemma:g-property} which tells that $\lambda = g^{-1}(K) \le T / K$, $\circle{2}$ is due to $w_t := \eta |S_{t-1}|^\alpha$ and $|S_{t-1}|\le N / 2^{t-1}$, and $\circle{3}$ due to $\log(1+x) \le x$.  

Putting all of them together, we know that 
\begin{equation}
    r^* \ge \max_{\{k_t\}} \frac{1}{|S_T|}P(\vx^* \in S_T) \ge \frac{2^T}{N}\exp\left(- \phi(\alpha, T)\frac{\eta T N^\alpha}{K}\right)  
\end{equation}
\end{proof}
\textbf{Optimal action sequence $\{k^*_t\}$}. From the proof, we could also write down the optimal action sequence that achieves the best reward: $k^*_t = w_t \log \left(1 + \frac{1}{\lambda w_t}\right)$, where $w_t := \eta |S_{t-1}|^\alpha$. Using Lemma~\ref{lemma:g-property}, we could compute the upper and lower bound estimation of $\lambda = g^{-1}(K)$. Here $\bar w := T \left(\sum_{t=1}^T w_t^{-1}\right)^{-1}$ be the Harmonic mean of $\{w_t\}$ and $w_{\max} := \max_{t=1}^T w_t$:
\begin{equation}
    \frac{\bar w^{-1}}{\exp(\bar w^{-1} K / T) - 1} \le \lambda \le \frac{w^{-1}_{\max}}{\exp(w^{-1}_{\max} K / T) - 1}
\end{equation}
With $\lambda$, we could compute approximate $\{k^*_t\}$. Here we make a rough estimation of $\{k^*_t\}$ if we terminate the algorithm when $|S_T|$ is still fairly large. This case corresponds to the setting $T = \beta \log_2 N$ where $\beta < 1$ and all $w_t \sim N^\alpha$. With $K = \Theta(N^\alpha)$ as in semi-parametric case, $\bar w^{-1} K = \Theta(1)$, $\exp(\bar w^{-1}K / T) - 1 \approx \bar w^{-1}K / T$ and $\lambda w_t \sim \log_2 N \gg 1$. Since $\log(1+x)\approx x$ for small $x$, we have $k^*_t \approx w_t \frac{1}{\lambda w_t} = 1/\lambda$, which is independent of $t$. Therefore, a constant amount of sampling at each stage is approximately optimal. 

\begin{observation}
\label{observation:isotropic-appendix}
If all $f_j$ are isotropic, $f_j(\vx) = \|\vx - \vc_j\|^2_2$, then $\Omega_P = \mathrm{ConvexHull}(\vc_1, \ldots, \vc_M)$.
\end{observation}
\begin{proof}
Consider $J(\vx;\mu) := \sum_{j=1}^M\mu_j f_j(\vx)$ where the weights $\mu_j \ge 0$ satisfies $\sum_{j}\mu_j = 1$. For brevity, we write the constraint as $\Delta := \{\mu: \mu_j \ge 0, \sum_j \mu_j = 1\}$.  

Now consider the Pareto Set $\Omega_P := \{\vx: \exists\mu\in \Delta:\ \nabla_\vx J(\vx; \mu) = 0\}$. We have the following:
\begin{eqnarray}
& & \nabla_\vx J(\vx;\mu) = 0 \\
& \Longleftrightarrow{} & \sum_{j}\mu_j \nabla_\vx f_j(\vx)=0 \\
& \Longleftrightarrow{} & \sum_{j}\mu_j(\vx - \vc_j) = 0 \\
& \Longleftrightarrow{} & \vx = \frac{\sum_{j}\mu_j \vc_j}{\sum_{j}\mu_j} = \sum_{j}\mu_j \vc_j
\end{eqnarray}
The last step is due to the fact that $\sum_{j}\mu_j = 1$. Therefore, for any $\vx \in \Omega_P$, $\vx$ is a convex combination of $\{\vc_1,\dots,\vc_M\}$ and thus $\vx \in \mathrm{ConvexHull}(\vc_1, \ldots, \vc_M)$. Conversely, for any $\vx \in \mathrm{ConvexHull}(\vc_1, \ldots, \vc_M)$, we know $\nabla_\vx J(\vx;\mu) = 0$ and thus $\vx \in \Omega_P$.  
\end{proof}

\begin{observation}
If $M = 2$ and $f_j(\vx) = (\vx - \vc_j)^\top H_j (\vx-\vc_j)$ where $H_j$ are positive definite symmetric matrices, then there exists $\vw_1 := H_2(\vc_2 - \vc_1)$ and $\vw_2:= H_1(\vc_1-\vc_2)$, so that for any $\vx \in \Omega_P$, $\vw_1^\top(\vx - \vc_1) \ge 0$ and $\vw_2^\top(\vx - \vc_2) \ge 0$. 
\end{observation}
\begin{proof}
Following Observation~\ref{obs:isotropic}, similarly we have for all $x\in\Omega_P, \sum_{j}\mu_j H_j(\vx - \vc_j) = 0$, which gives:
\begin{equation}
 \vx = \left(\sum_{j}\mu_j H_j\right)^{-1}\sum_{j}\mu_j H_j \vc_j  \label{eq:x-appendix} 
\end{equation}
Note that this is an expression of the Pareto Set $\Omega_P$. 

Let $A_{j} := (\sum_{j}\mu_j H_j)^{-1}\mu_j H_j$. Then $\sum_{j}A_{j}=I$. Note that while $\sum_{j}\mu_j H_j$ and $(\sum_{j}\mu_j H_j)^{-1}$ are positive definite matrix. $A_j$ may not be. 

Let $M := \sum_{i}\mu_iH_i$. Since $\mu \in \Delta$, $M$ is a PD matrix. Note that we have 
\begin{eqnarray}
\sum_{j}\mu_jH_j\vc_j &=& \sum_{j}\mu_jH_j\vc_j - \sum_{j}\mu_jH_j\vc_k + \sum_{j}\mu_jH_j\vc_k \\
&=& \sum_{j\neq k}\mu_j H_j (\vc_j - \vc_k) + M\vc_k
\end{eqnarray}
Using Eqn.~\ref{eq:x-appendix}, we know that $\vx = M^{-1}\sum_{j}\mu_j H_j \vc_j  = \vc_{k} + M^{-1}\sum_{j\neq k}\mu_jH_j(\vc_j - \vc_k)$. 

For $M = 2$, we have $\vx = \vc_2 + M^{-1}\mu_1H_1(\vc_1 - \vc_2)$. So we have
\begin{eqnarray}
(\vc_1 - \vc_2)^\top H_1 \vx &=& (\vc_1 - \vc_2)^\top H_1 \vc_2 + (\vc_1 - \vc_2)^\top H_1M^{-1}H_1(\vc_1 - \vc_2) \\
& \ge & (\vc_1 - \vc_2)^\top H_1 \vc_2
\end{eqnarray}
This is because $(\vc_1 - \vc_2)H_1M^{-1}H_1(\vc_1 - \vc_2) \geq 0$ since $H_1M^{-1}H_1$ is a PSD matrix. Therefore, let $\vw_2 := H_1 (\vc_1 - \vc_2)$ and we have $\vw_2^\top (\vx - \vc_2) \ge 0$, which is independent of $\mu \in \Delta$. This means it holds for any $\vx \in \Omega_P$. 

Let $\vw_1 = H_2(\vc_2 - \vc_1)$, then similarly we have $\vw_1^\top (\vx - \vc_1) \ge 0$ for all $\vx \in \Omega_P$. 
\end{proof}

\section{Quality Indicators Comparison}
\label{app:qi}
\begin{table}[H]
\centering
\caption{The review of different scalarizing methods.}
\label{tab:scalarizing_review}
\small
\begin{tabular}{|l|c|c|c|c|}
\hline
Quality Indicator & Convergence & Uniformity & Spread & No reference set required \\ \hline
HyperVolume       & $\surd$        & $\surd$       & $\surd$   & $\surd$                      \\ \hline
GD                & $\surd$            &        &        &                           \\ \hline
IGD               & $\surd$        & $\surd$       & $\surd$   &                           \\ \hline
MS                &             &            & $\surd$   &                           \\ \hline
S                 &             & $\surd$       &        &                           \\ \hline
ONVGR             & $\surd$        &            &        &                           \\ \hline
\end{tabular}
\end{table}

Generational Distance(GD)~\citep{gd} measures the distance the pareto frontier of approximation samples and true pareto frontier, which requires prior knowledge of true pareto frontier and only convergence is considered. IGD~\citep{igd} is improved version of GD. IGD calculates the distance the points on true pareto frontier to the closest point on pareto frontier of current samples. Inverted Generational Distance(IGD) satisfies all three evaluation metrics of QI but requires true pareto frontier which is hardly to get in real-world problem. Maximum Spread(MS)~\citep{maxspread} computes the distance between the farthest two points of samples to evaluate the spread. Spacing(S)~\citep{spacing} measures how close the distribution of pareto frontier of samples is to uniform distribution. Overall Non-dominated Vector Generation and Ratio(ONVGR) is the ratio of number of samples in true pareto-frontier. The table~\ref{tab:scalarizing_review} demonstrates the good characteristics of each quality indicators.

\section{End-to-end LaMOO Pseudocode}
Below we list the pseudocode for the end-to-end workflow of LaMOO in Algorithm~\ref{alg:lamoo_full}. Specfically, it includes search space partition in \textbf{Function Split}. Node(promising region) selection in \textbf{Function Select}, and new samples generation in \textbf{Function Sample}. 

\algdef{SE}[SUBALG]{Indent}{EndIndent}{}{\algorithmicend\ }%
\algtext*{Indent}
\algtext*{EndIndent}


\begin{algorithm}[H]
    \small
	\caption{\ours{} Pseudocode.}
	\label{alg:lamoo_full}
	\begin{algorithmic}[1]
	\State {\bfseries Inputs:} Initial $D_0$ from uniform sampling, sample budget $T$.
	\For{$t = 0, \dots, T$}
	\State Set $\mathcal{L} \leftarrow \{\Omega_\root\}$ (collections of regions to be split).

    \State $\mathcal{V}, v, n \leftarrow \textbf{Split}$($\mathcal{L}$, $D_t$)
    \State $k \leftarrow $ $\textbf{Select}$($\mathcal{C}_{p}$, $D_t$)
    \State $D_{t+1} \leftarrow$ $\textbf{Sample}$($k$)

	\EndFor\\
	
	\State {\bfseries Function} $\textbf{Split}$($\mathcal{L}$, $D_t$)
	\Indent
    \While{$\mathcal{L} \neq \emptyset$}
	\State $\Omega_j \leftarrow \mathrm{pop\_first\_element}(\mathcal{L}),\ \  D_{t,j} \leftarrow D_t \cap \Omega_j, \ \ n_{t,j} \leftarrow |D_{t,j}|$. 
	\State Compute dominance number $o_{t,j}$ of $D_{t,j}$ using Eqn.~\ref{eq:dominance} and train SVM model $h(\cdot)$.
	\State \textbf{If} $(D_{t,j}, o_{t,j})$ is splittable by SVM, \textbf{then} $\mathcal{L} \leftarrow \mathcal{L} \cup \mathrm{Partition}(\Omega_j, {h(\cdot)})$.
	
	\EndWhile
	\EndIndent

	\State
    \State {\bfseries Function} $\textbf{Select}$($C_p$, $D_t$)
    \Indent
    \For{$k = \root$, $k$ is not leaf node}
	    \State $D_{t,k} \leftarrow D_t \cap \Omega_k, \ \ v_{t,k} \leftarrow \mathrm{HyperVolume}(D_{t,k}),\ \ n_{t,k} \leftarrow$ $|D_{t,k}|$.
	    \State $k \leftarrow \displaystyle\arg\max_{c\ \in \ \mathrm{children}(k)} \mathrm{UCB}_{t,c}$, where $\mathrm{UCB}_{t,c} := v_{t,c} + 2 C_p \sqrt{\frac{2\log(n_{t,k}}{n_{t,c}}}$
	\EndFor
	\State \textbf{return} $k$
    \EndIndent
    \State
    \State {\bfseries Function} $\textbf{Sample}$($k$)
    \Indent
    \State $D_{t+1} \leftarrow$ $D_{t} \cup D_{\mathrm{new}}$, where $D_{\mathrm{new}}$ is drawn from $\Omega_{k}$ based on qEHVI or CMA-ES.  
    \State \textbf{return} $D_{t} \cup D_{new}$
    \EndIndent
  \end{algorithmic}
\label{alg:lamoo_full}
\end{algorithm}

\section{Exploration factor($C_p$) setup with unknown maximum Hypervolume}

\begin{figure}[H]

\centering 
\subfloat[][BraninCurrin]{\includegraphics[height=1.33in]{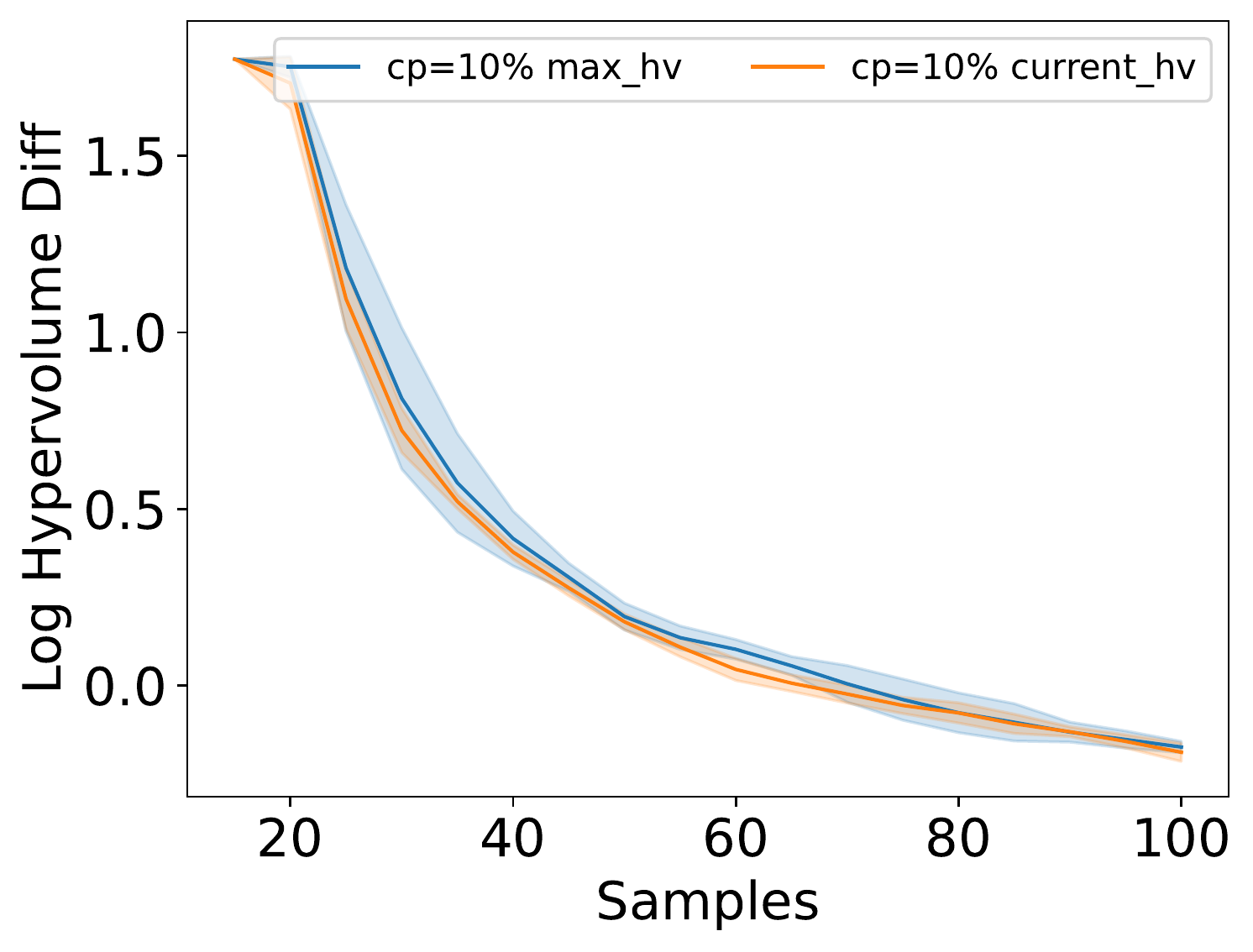}}
\subfloat[][VehicleSafety]{\includegraphics[height=1.33in]{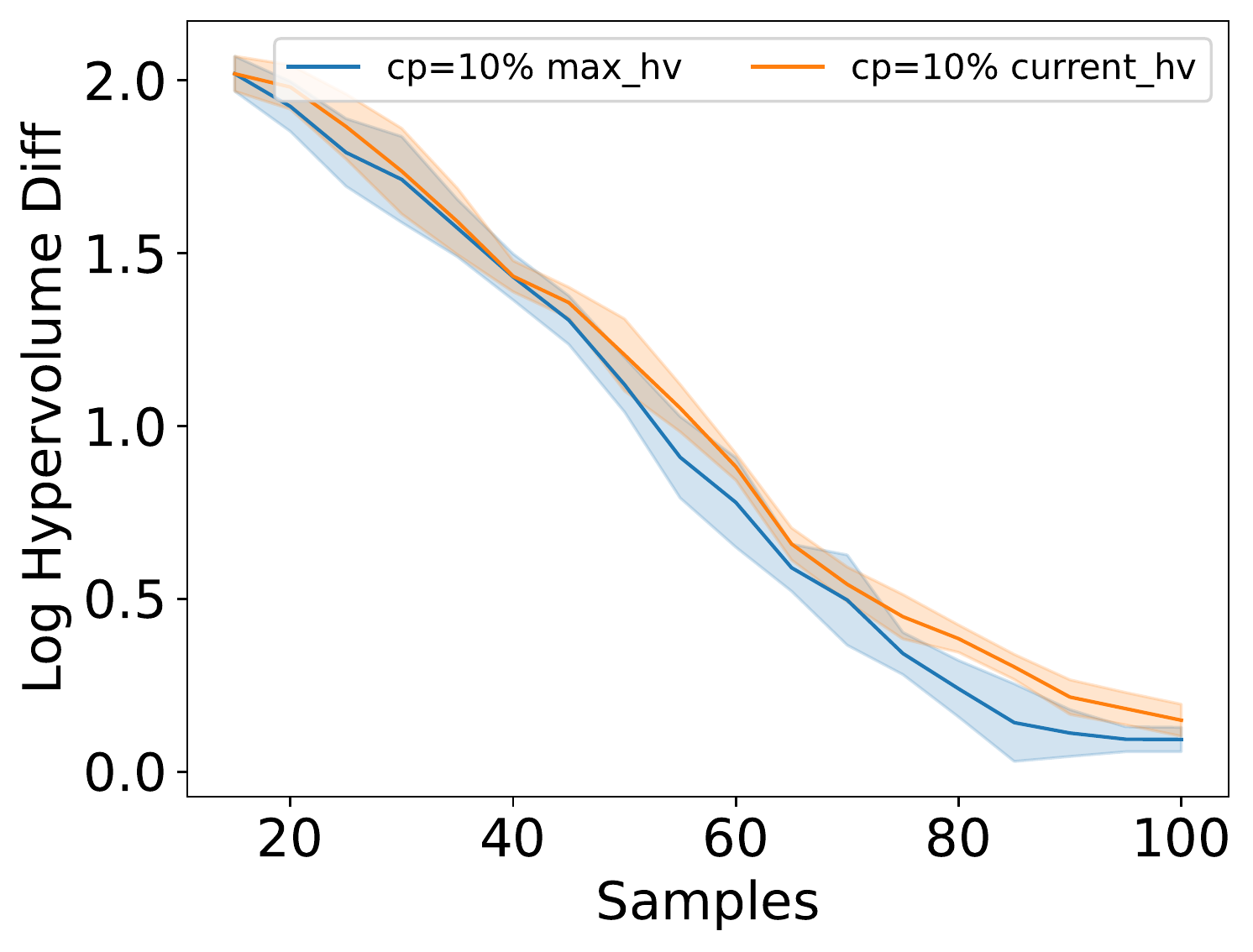}}  
\subfloat[][Nasbench201]{\includegraphics[height=1.33in]{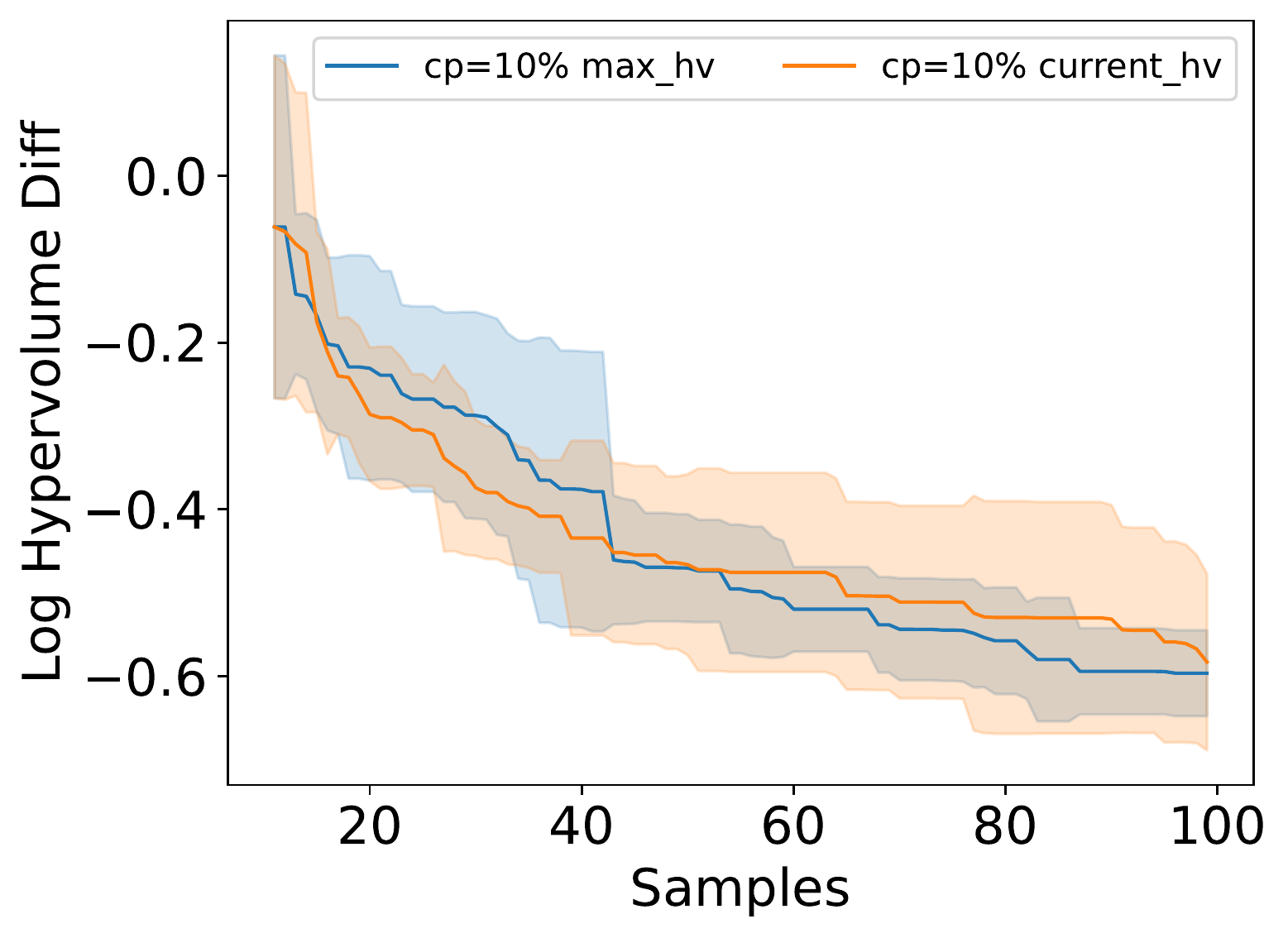}}  
\quad
\caption{\small Sampling with static $C_p$(10\% of $HV_{max}$) and dynamic $C_p$((10\% of $HV_{current}$))}
\label{fig:cp_test}
\end{figure}

As we mentioned in the paper, a "rule of thumb" is to set the $C_p$ to be roughly 10\% of the maximum hypervolume HVmax. If HVmax is unknown, $C_p$ can be dynamically set to 10\% of the hypervolume of current samples in each search iteration. The figures below demonstrate the difference between 10\% HVmax and 10\% current hypervolume in three problems(Branin-Currin, VehicleSafety, and  Nasbench201 from left to right). The final performances by using 10\% HVmax and 10\% current hypervolume are similar.

\section{Wall clock time in different problems}

\begin{figure}[H]

\centering 
\subfloat[][BraninCurrin]{\includegraphics[height=1.33in]{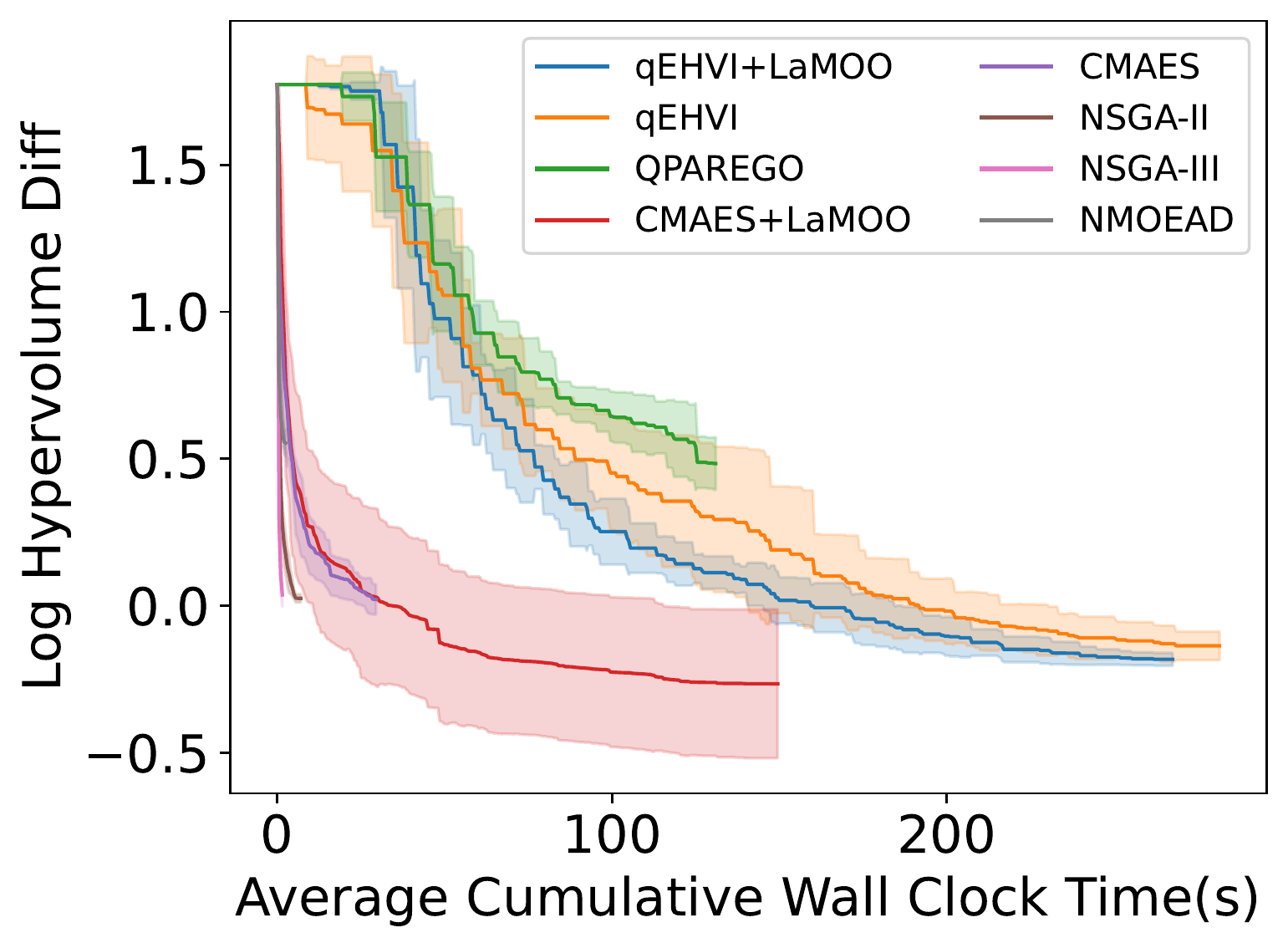}}
\subfloat[][VehicleSafety]{\includegraphics[height=1.33in]{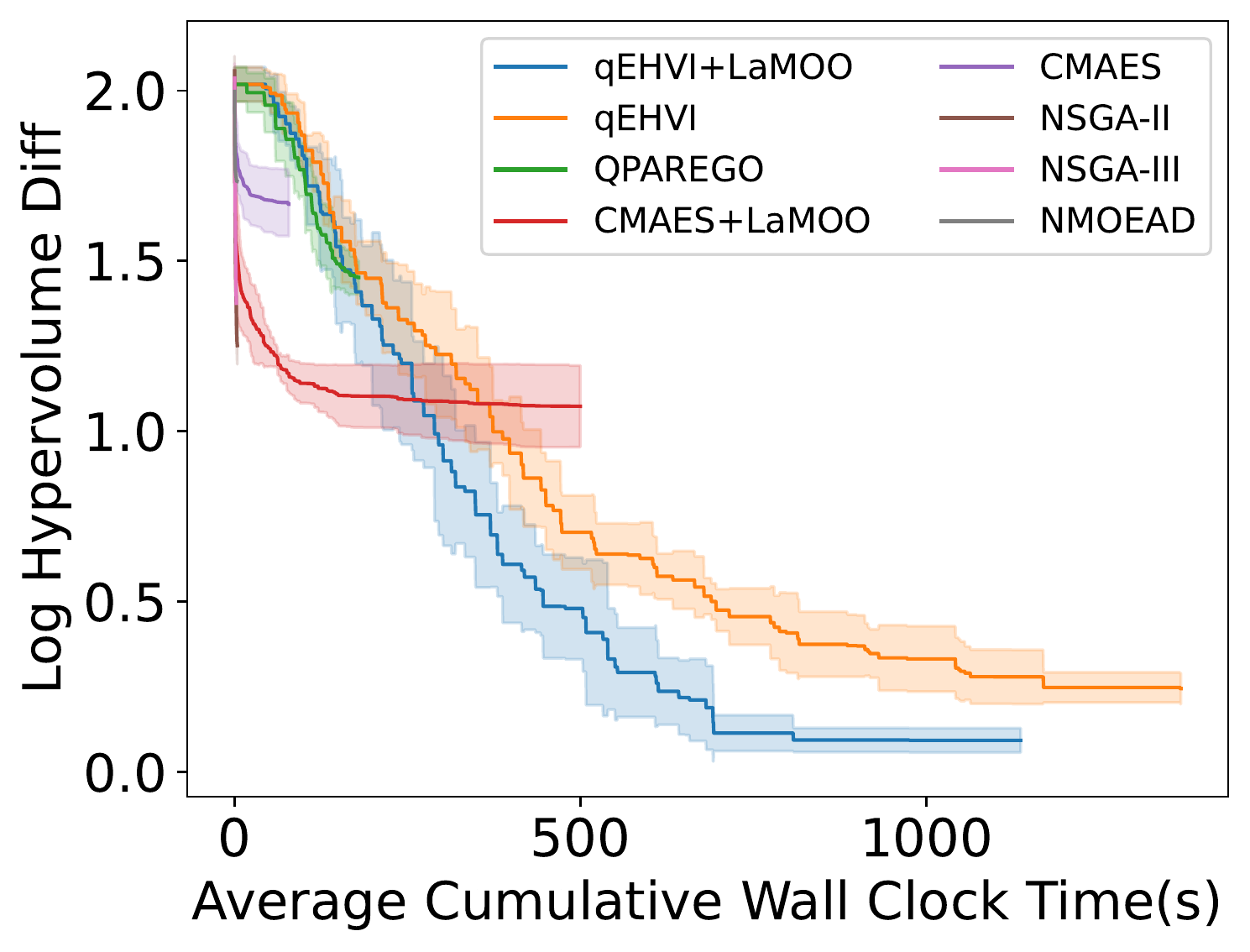}}  
\subfloat[][Nasbench201]{\includegraphics[height=1.33in]{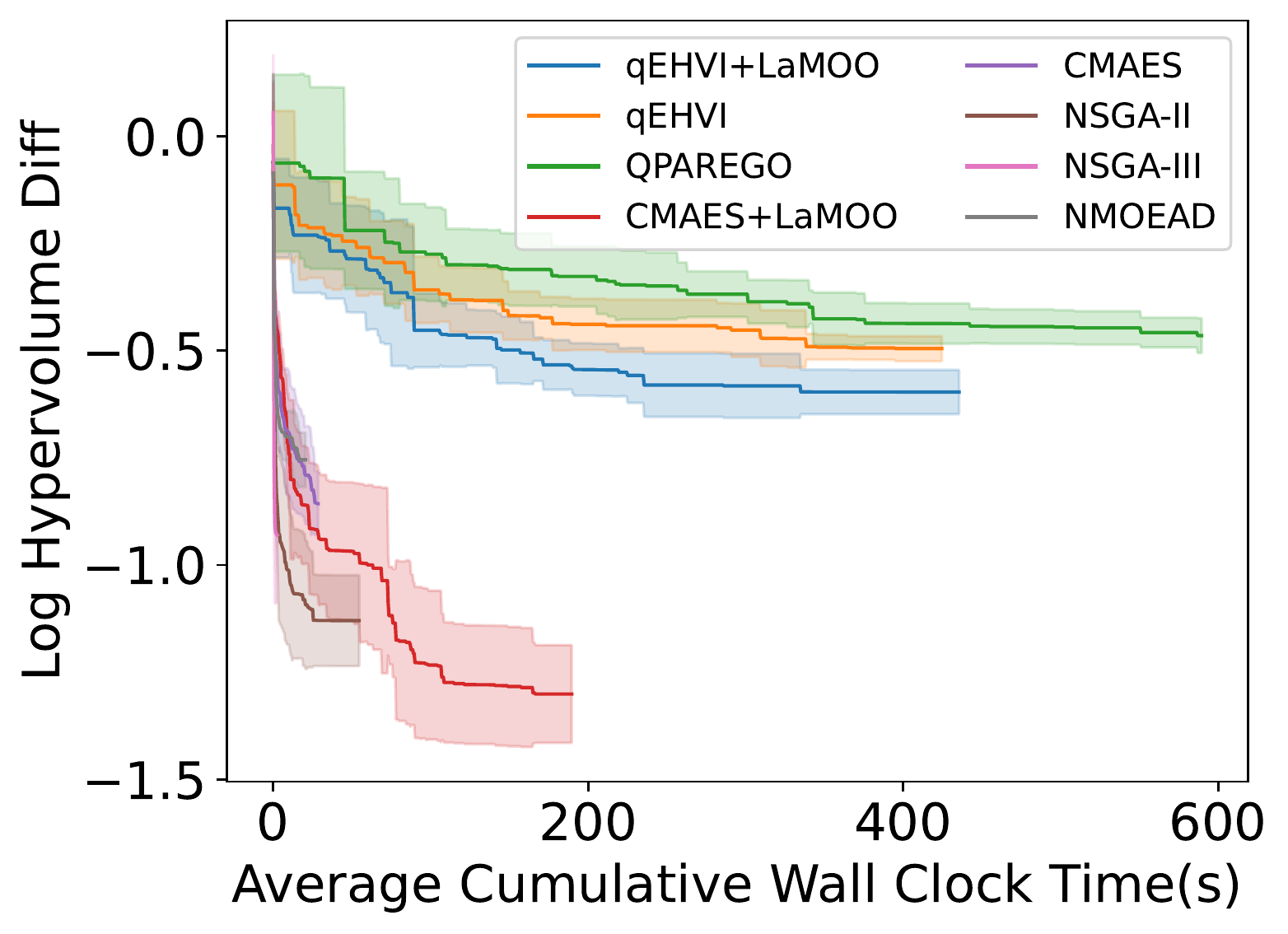}}  
\quad
\caption{\small  Wall clock time in different problems}
\label{fig:wall_clock}
\end{figure}

Fig.~\ref{fig:wall_clock} shows the wall clock time of different search algorithms in BraninCurrin\citep{bc_func}, VehicleSafefy~\citep{vehicle_safety} and Nasbench201~\citep{nasbench201}.

\section{Details of Benchmark Problems}
\subsection{Problem description}
\textbf{BraninCurrin~\citep{bc_func}}:
\begin{align*}
f^{(1)}(x_1, x_2) &= (15x_2 - \frac{5.1 * (15x_1 - 5)^2} {4\pi^2} + \frac{75x_1 - 25}{\pi} - 5)^2 + (10 - \frac{10}{8\pi}) * \cos(15x_1 - 5)\\
f^{(2)}(x_1, x_2) &= \bigg[1 - \exp\bigg(-\frac{1} {(2x_2)}\bigg)\bigg] \frac{2300 x_1^3 + 1900x_1^2 + 2092 x_1 + 60}{100 x_1^3 + 500x_1^2 + 4x_1 + 20}\\
\end{align*}
where $x_1, x_2 \in [0,1]$.

\textbf{VehicleSafefy~\citep{vehicle_safety}}:
\begin{align*}
    f_1(\bm x) &=1640.2823 + 2.3573285x_1 + 2.3220035x_2 + 4.5688768x_3 + 7.7213633x_4 + 4.4559504x_5\\
    f_2(\bm x) &= 6.5856+ 1.15x_1 - 1.0427x_2 + 0.9738x_3+ 0.8364x_4 - 0.3695x_1x_4 + 0.0861x_1x_5\\
    &+ 0.3628x_2x_4 +  0.1106x_1^2 - 0.3437x_3^2 + 0.1764x_4^2\\
    f_3(\bm x) &= -0.0551 + 0.0181x_1 + 0.1024x_2 + 0.0421x_3 - 0.0073x_1x_2 + 0.024x_2x_3-0.0118x_2x_4\\ 
    &- 0.0204x_3x_4 - 0.008x_3x_5 - 0.0241x_2^2 + 0.0109x_4^2\\
\end{align*}
where $\bm x \in [1,3]^5$.

\textbf{Nasbench201~\citep{nasbench201}}:

\begin{figure}[ht]
\centering 
\includegraphics[width=0.80\columnwidth]{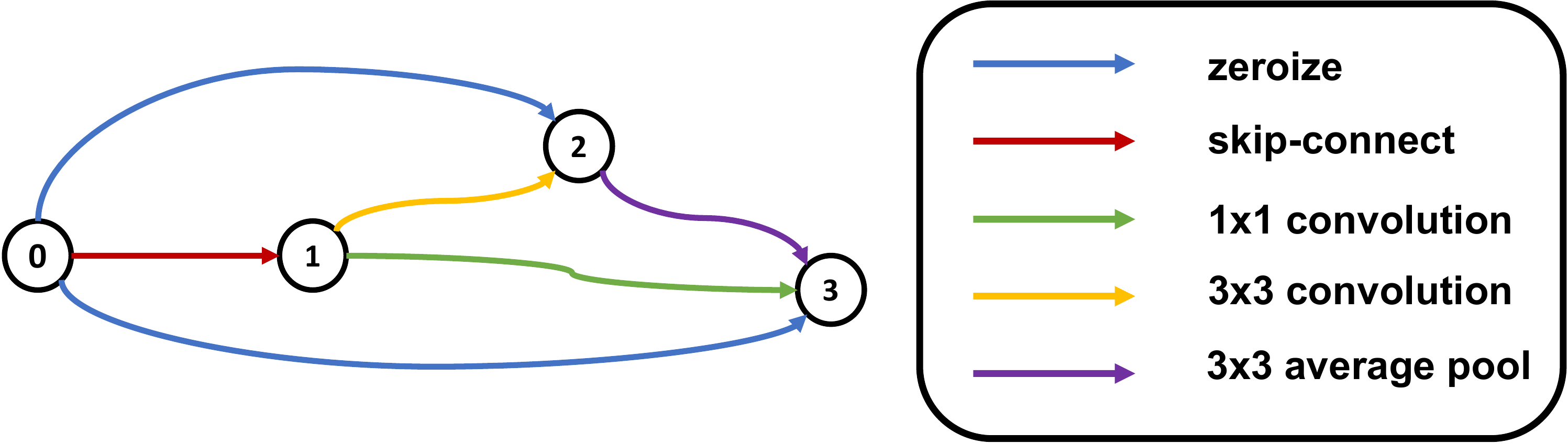}
\caption{\small A general architecture of Nasbench201}
\label{fig:nasbench201_arch}
\end{figure}

In Nasbench201, the architectures are made by stacking the cells together. The difference among architectures in Nasbench201 is the design of the cells, see fig~\ref{fig:nasbench201_arch}. Specifically, each cell contains 4 nodes, and there is a particular operation connecting to two nodes including zeroize, skip-connect, 1x1 convolution, 3x3 convolution, and 3x3 average pooling. Therefore, there are $C^{2}_{4}$ = 6 edges in a cell and ${5^6}$ =15625 unique architectures in Nasbench201. According to this background, Each architecture can be encoded into a 6-dimensional vector with 5 discrete numbers (i.e., 0, 1, 2, 3, 4 that corresponds to zeroize, skip-connect, 1x1 convolution, 3x3 convolution, and 3x3 average pooling). 
\begin{align*}
    f_1(\bm x) &= Accuracy{(\bm x)}\\
    f_2(\bm x) &= \#FLOPs{(\bm x)}\\
\end{align*}
where $\bm x \in \{0, 1, 2, 3, 4\}^6$.

\textbf{DTLZ2~\citep{dtlz}}:
\begin{align*}
    f_1(\bm x) &= (1+ g(\bm x_M))\cos\big(\frac{\pi}{2}x_1\big)\cdots\cos\big(\frac{\pi}{2}x_{M-2}\big) \cos\big(\frac{\pi}{2}x_{M-1}\big)\\
    f_2(\bm x) &= (1+ g(\bm x_M))\cos\big(\frac{\pi}{2}x_1\big)\cdots\cos\big(\frac{\pi}{2}x_{M-2}\big) \sin\big(\frac{\pi}{2}x_{M-1}\big)\\
    f_3(\bm x) &= (1+ g(\bm x_M))\cos\big(\frac{\pi}{2}x_1\big)\cdots\sin\big(\frac{\pi}{2}x_{M-2}\big)\\
    \vdots\\
     f_M(\bm x) &= (1+ g(\bm x_M))\sin\big(\frac{\pi}{2}x_1\big)
\end{align*}
where $g(\bm x) = \sum_{x_i \in \bm x_M} (x_i - 0.5)^2, \bm x \in [0,1]^d,$ and $\bm x_M$ represents the last $d - M +1$ elements of $\bm x$.

\subsection{Visualization of Pareto-Frontier for Benchmark Problems}
\begin{figure}[H]

\centering 
\subfloat[][BraninCurrin]{\includegraphics[height=1.33in]{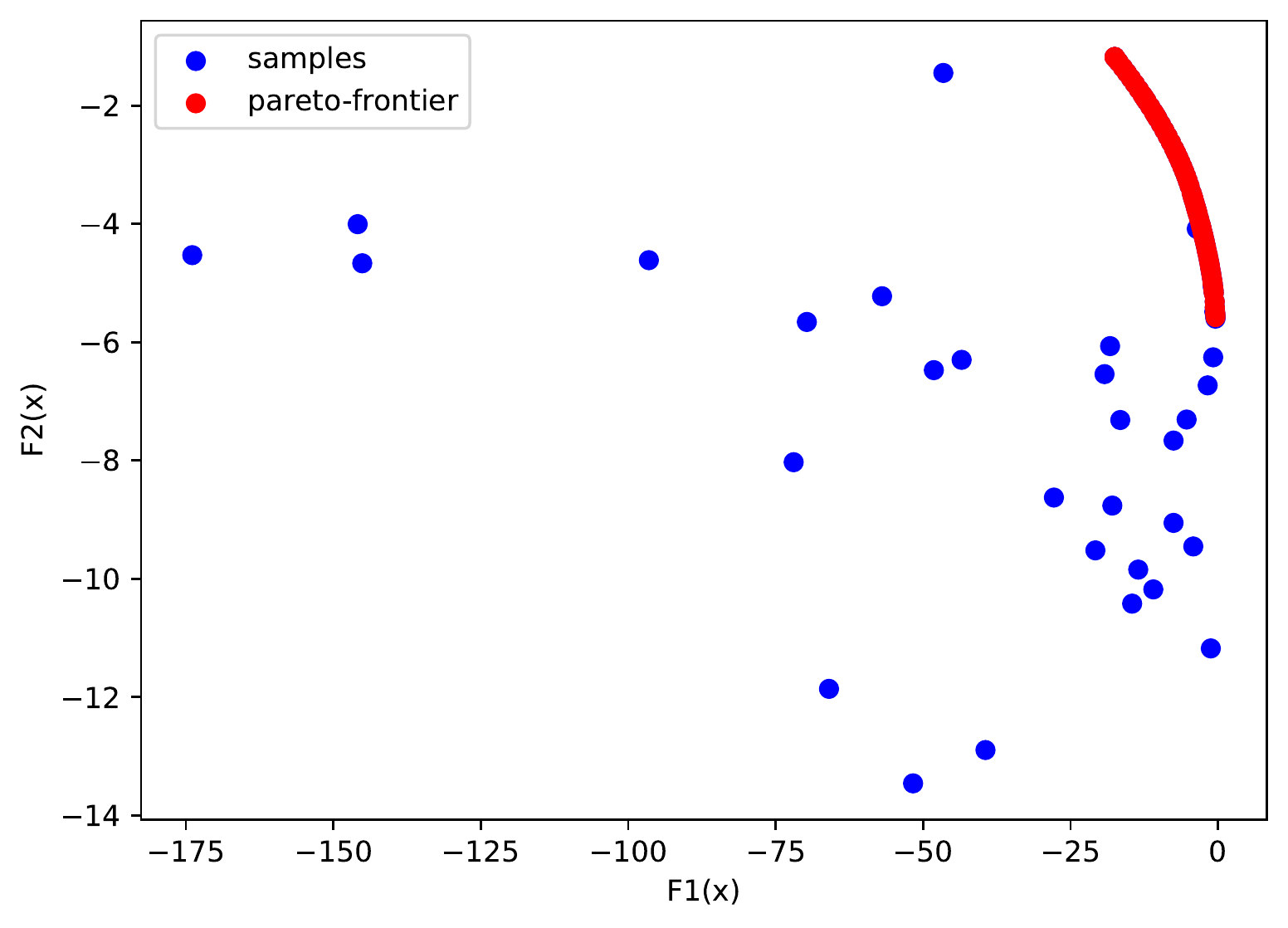}}
\subfloat[][VehicleSafety]{\includegraphics[height=1.56in]{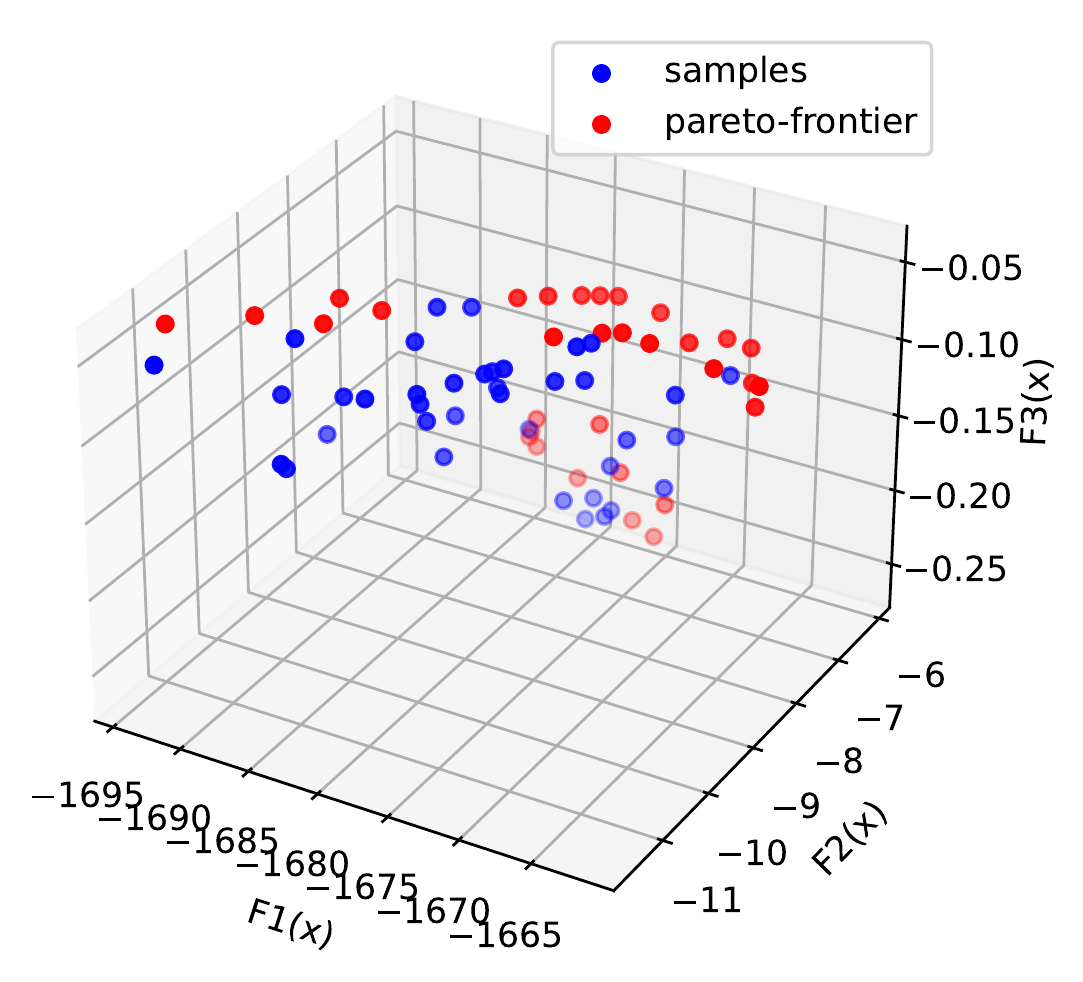}}  
\subfloat[][Nasbench201]{\includegraphics[height=1.33in]{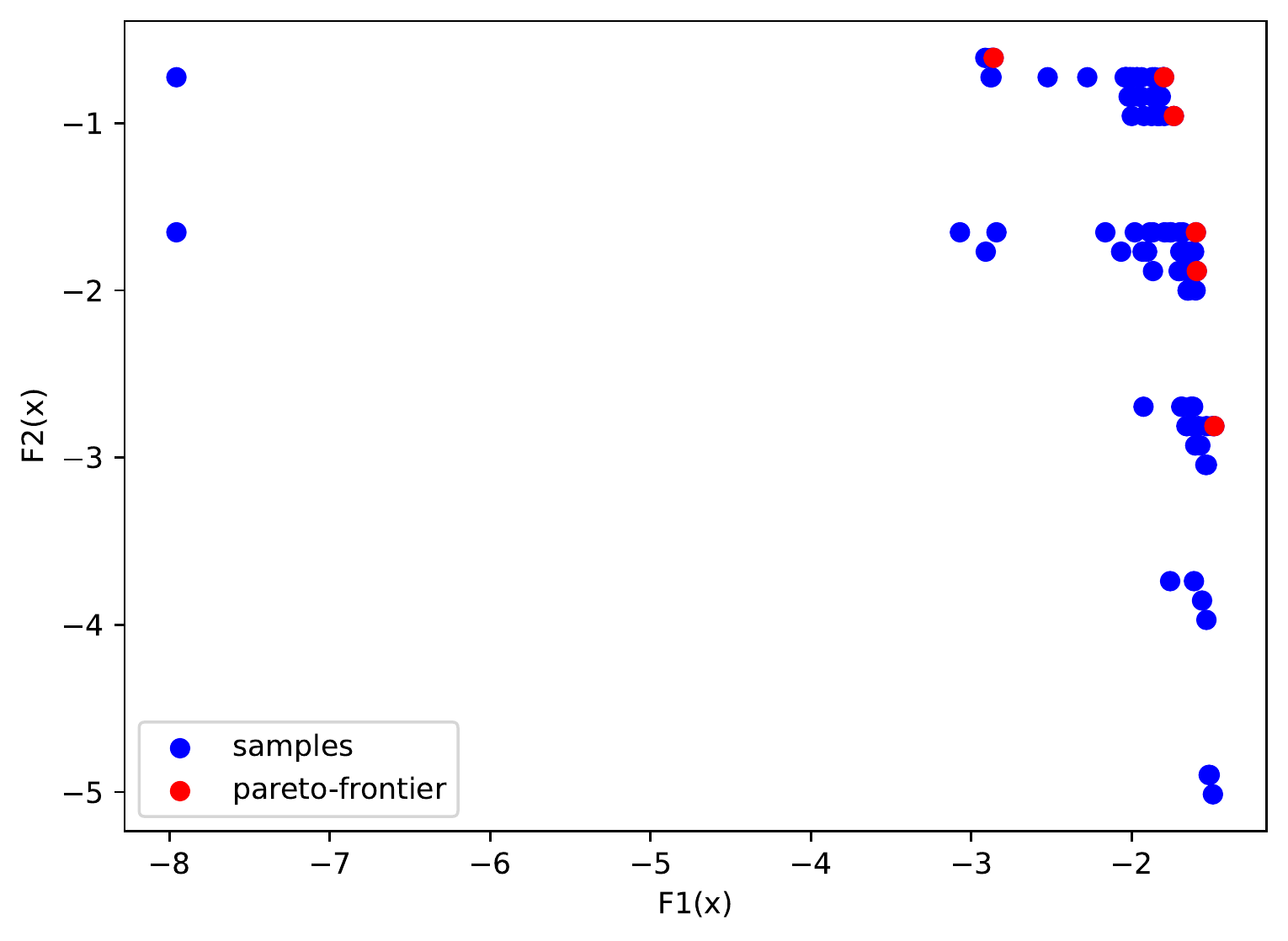}}  
\quad
\caption{\small  Visualization of Pareto-frontier in BraninCurrin, VehicleSafety as well as Nasbench201.}
\label{fig:pareto}
\end{figure}

\subsection{Reference Points}

\begin{table}[H]

\centering

\begin{tablenotes}
    \centering
     \item $^\dagger$: $M$ represents the number of objectives.
\end{tablenotes}
\begin{tabular}{|c|c|}
\hline
Problem            & Reference Point                         \\ \hline
BraninCurrin       & (18.0, 6.0)                             \\ \hline
VehicleSafety      & (1864.72022, 11.81993945, 0.2903999384) \\ \hline
Nasbench201        & (-3.0, -6.0)                            \\ \hline
DTLZ2              & (1.1, .., 1.1) $\in \mathbb {R^M} ^{\dagger}$                         \\ \hline
Molucule Discovery & (0.0, ..., 0.0) $\in \mathbb {R^M} ^{\dagger}$                       \\ \hline
\end{tabular}
\caption{ The reference points for all problems in this work.}
\label{tab:ref}
\end{table}

The reference point $R \in \rr^M$ is defined to measure the hypervolume of a problem. Different reference point would result in a different hypervolume. The details can be found at sec ~\ref{sec:intro}.
Table~\ref{tab:ref} elaborates the reference points in the problems throughout the paper. 

\subsection{Maximum Hypervolume of each problem}

\begin{table}[H]
\centering
\begin{tabular}{|c|c|}
\hline
Problem              & Maximum Hypervolume \\ \hline
BraninCurrin         & 59.36011874867746   \\ \hline
VehicleSafety        & 246.81607081187002  \\ \hline
Nasbench201          & 8.06987476348877    \\ \hline
DTLZ2(2 objectives)  & 1.4460165933151778  \\ \hline
DTLZ2(10 objectives) & 2.5912520655298095  \\ \hline
Molucule Discovery   & N/A                 \\ \hline
\end{tabular}
\caption{ The maximum hypervolume for all problems in this work.}
\label{tab:maxhv}
\end{table}

Table~\ref{tab:maxhv} elaborates the observed maximal hypervolume in the problems throughout the paper. We used these value to calculate the log hypervolume difference in fig~\ref{fig:small_problems} and  fig~\ref{fig:DTLZ2}.

\section{Experiments Setup}

Experiment details: For small-scale problems(i.e. Branin-Currin, VehicleSafety, and  Nasbench201) and DTLZ2 with 2 and 10 objectives. We randomly generate 10 samples as the initialization. For multi-objective molecule discovery, the number of initial samples is 150. In each iteration, we update 5 batched samples(q value) for all search algorithms. 

Hyperparameters of LAMOO: For all problems, we leverage polynomials as the kernel type of SVM and the degree of the polynomial kernel function is set to 4. The minimum samples in the leaf of MCTS is 10. The cp is roughly set to 10\% of maximum of hypervolume(i.e. Branin-Currin -> 5, VehicleSafety -> 20, Nasbench201 -> 6, DTLZ2(2 objectives) -> 0.1, DTLZ2(10 objectives) -> 0.25,  molecule discovery(2 objectives) -> 0.03, molecule discovery(3 objectives) -> 0.2, molecule discovery(4 objectives) -> 0.06). 

Hyperparameters of qEHVI and qParEGO: The number of q is set to 5.  The acquisition function is optimized with L-BFGS-B (with a maximum of 200 iterations). In each iteration, 256 raw samples used for the initialization heuristic are generated to be selected by the acquisition function. In original work~\cite{qehvi}, they used 1024 raw samples but we decrease this number to 256 to sample budget of all methods for comparison, which speeds up the search but may lead to lower performance such as vehiclesafty problem in fig.~\ref{fig:small_problems}. As the same claim in Daulton et al. (2020), each generated sample is modeled with an independent Gaussian process with a Matern 5/2 ARD kernel.

\section{Verification of LAMOO on Many-objective Problems}

\begin{figure}[ht]

\centering 
\includegraphics[width=1.00\columnwidth]{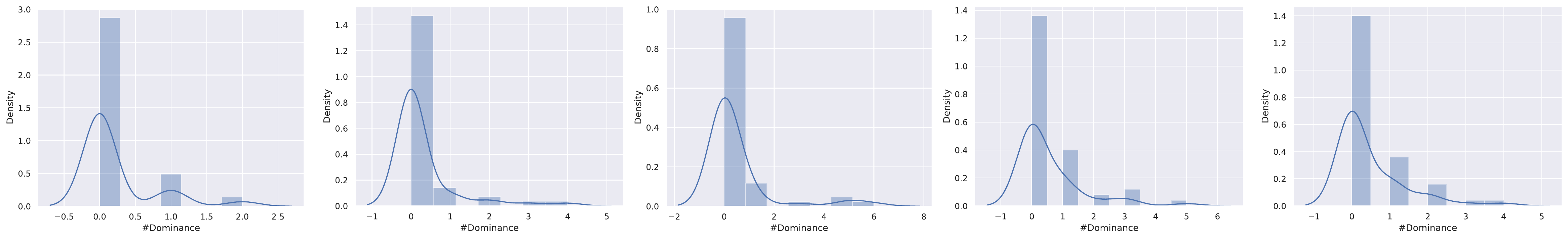}
\vspace{-0.1in}
 \caption{\small  Dominance number distribution with 50 random samples on DTLZ2(10 objectives)}
\label{fig:many_test1}
\end{figure}

\begin{figure}[ht]

\centering 
\includegraphics[width=0.50\columnwidth]{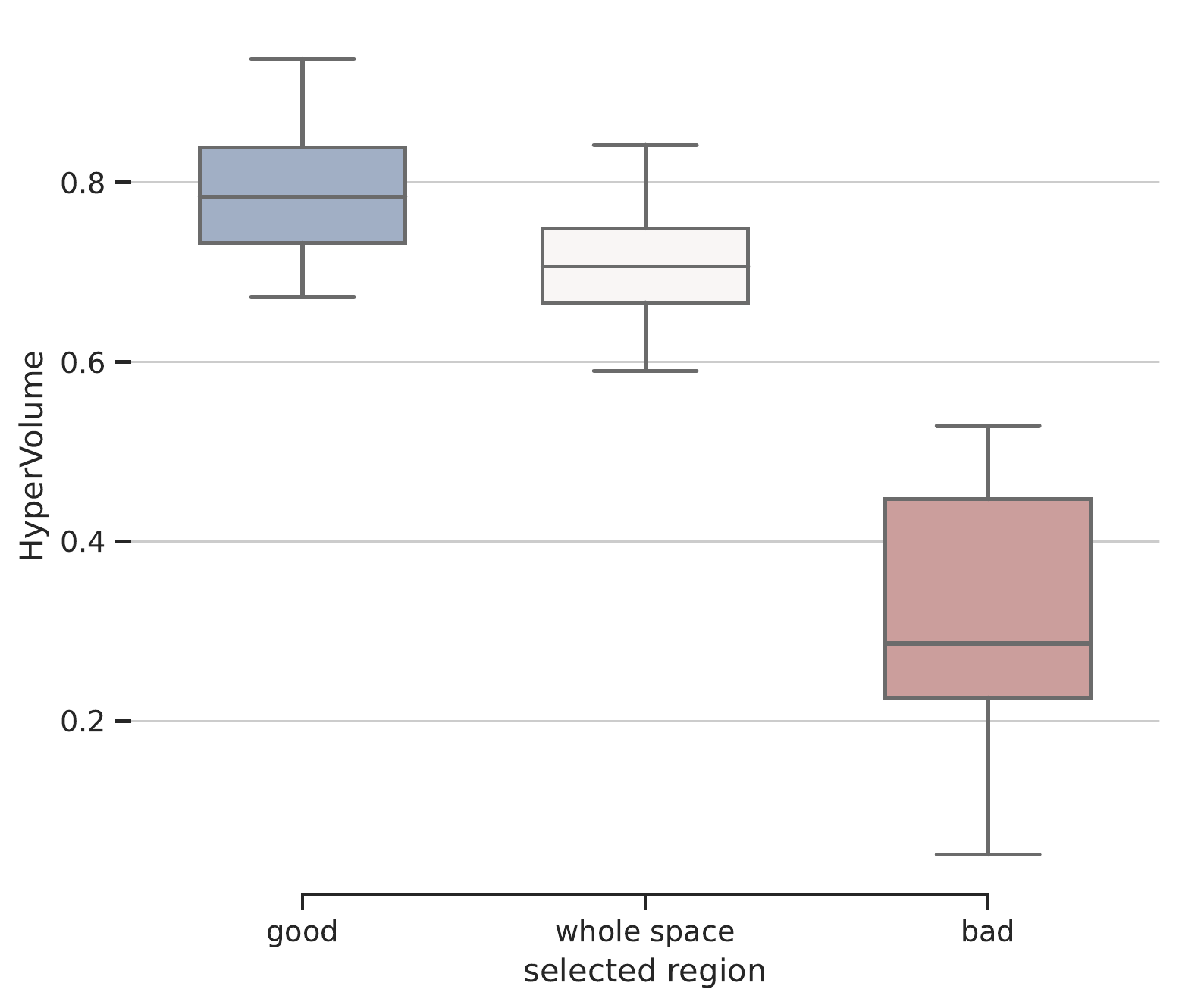}
\vspace{-0.1in}
 \caption{\small  The range of hypervolume for 50 samples randomly generated from different regions in DTLZ2(10 objectives). We generate 25 times of 50 samples in total. }
\label{fig:many_test2}
\end{figure}

While it is theoretically hard to label the samples into good and bad based on their dominance number in many-objective problems due the the lack of dominance pressure(All samples are non-dominated with each other). If number of objective is not too large(i.e. $M\leq 10$), the samples can be still split by dominance number. Given the problem(DTLZ2 with 10 objectives) shown in fig~\ref{fig:DTLZ2}, we randomly generate 50 samples in the search space and draw the dominance distribution of them(see fig~\ref{fig:many_test1}). We did this experiment 5 times.

We then partition the search space by a SVM classifier based on the labeled samples into “good” and “bad”, and randomly generate 50 samples in “good region”, “bad region”, and whole space, respectively. We did this process 5 times with different initial samples. Fig.~\ref{fig:many_test2} shows the range of hypervolume of the samples generated from good regions, the whole space, and bad regions.  From the figure, we can see that the hypervolume of samples generated from good regions are significantly higher than others.

\section{Computational complexity analysis of ~\ours{}}

Here is a detailed breakdown of the computational complexity of our algorithm (Alg.~\ref{alg:lamoo})

Line.6: Compute dominance: $O(M{N_{node}}^{2})$ where $N_{node}$ is the number of samples in the node and $M$ is the number of dimensions. 

Line 7: Time complexity of SVM : $O({N_{node}}^{2})$ where $N_{node}$ is the number of samples in the node.

Line 10: Hypervolume: $O(N^{\frac{M}{2}} + N\log{N})$($M>3$)~\citep{hv_comp1} or $O(N\log{N})$($M\leq3$)~\citep{hv_comp2}, where $N$ is number of searched samples in total and $D$ is the number of dimensions. 

Total time complexity: $\sum_{i=1}^{t} O({N_{i}}^{2})$($M<3$), where $t$ is the total number of nodes. $\sum_{i=1}^{t} O(N^{\frac{M}{2}} + N_i\log{N_i})$ ($M>3$), where $t$ is the total number of nodes. 

When there are more than 3 objectives ($M>3$), HV computation is the dominant factor. When $M\le 3$, the optimization cost of SVM is the dominant factor.

\section{Variation of ~\ours{} with a cheaper overhead}
\algdef{SE}[SUBALG]{Indent}{EndIndent}{}{\algorithmicend\ }%
\algtext*{Indent}
\algtext*{EndIndent}


\def\root{\mathrm{root}}

\begin{algorithm*}[h]
    \small
	\caption{ \ours{} Pseudocode with leaf based selection.}
	\begin{algorithmic}[1]
	\State {\bfseries Inputs:} Initial $D_0$ from uniform sampling, sample budget $T$.
	\For{$t = 0, \dots, T$}
	\State Set $\mathcal{L} \leftarrow \{\Omega_\root\}$ (collections of regions to be split). 
	\While{$\mathcal{L} \neq \emptyset$}
	\State $\Omega_j \leftarrow \mathrm{pop\_first\_element}(\mathcal{L}),\ \  D_{t,j} \leftarrow D_t \cap \Omega_j, \ \ n_{t,j} \leftarrow |D_{t,j}|$. 
	\State Compute dominance number $o_{t,j}$ of $D_{t,j}$ using Eqn.~\ref{eq:dominance} and train SVM model $h(\cdot)$.
	\State \textbf{If} $(D_{t,j}, o_{t,j})$ is splittable by SVM, \textbf{then} $\mathcal{L} \leftarrow \mathcal{L} \cup \mathrm{Partition}(\Omega_j, {h(\cdot)})$.
	\EndWhile
	\For{$k = \root$, $k$ is not leaf node}
	    \State $D_{t,k} \leftarrow D_t \cap \Omega_k, \ \,\ \ n_{t,k} \leftarrow$ $|D_{t,k}|$.
	\EndFor
	
	\For{$l$ is leaf node}
	    \State $v_{t,l} \leftarrow \mathrm{HyperVolume}(D_{t,l})$
	\EndFor
	\State $k \leftarrow \displaystyle\arg\max_{l\ \in \ \mathrm{leaf\ nodes}} \mathrm{UCB}_{t,l}$, where $\mathrm{UCB}_{t,l} := v_{t,l} + 2 C_p \sqrt{\frac{2\log(n_{t,l})}{n_{t,p}}}$, where $p$ is the parent of $l$.

	\State $D_{t+1} \leftarrow$ $D_{t} \cup D_{\mathrm{new}}$, where $D_{\mathrm{new}}$ is drawn from $\Omega_{k}$ based on qEHVI or CMA-ES.  
    \EndFor

  \end{algorithmic}
\label{alg:lamoo_leaf}
\end{algorithm*}

\begin{figure}[H]

\centering 
\subfloat[][BraninCurrin]{\includegraphics[height=1.33in]{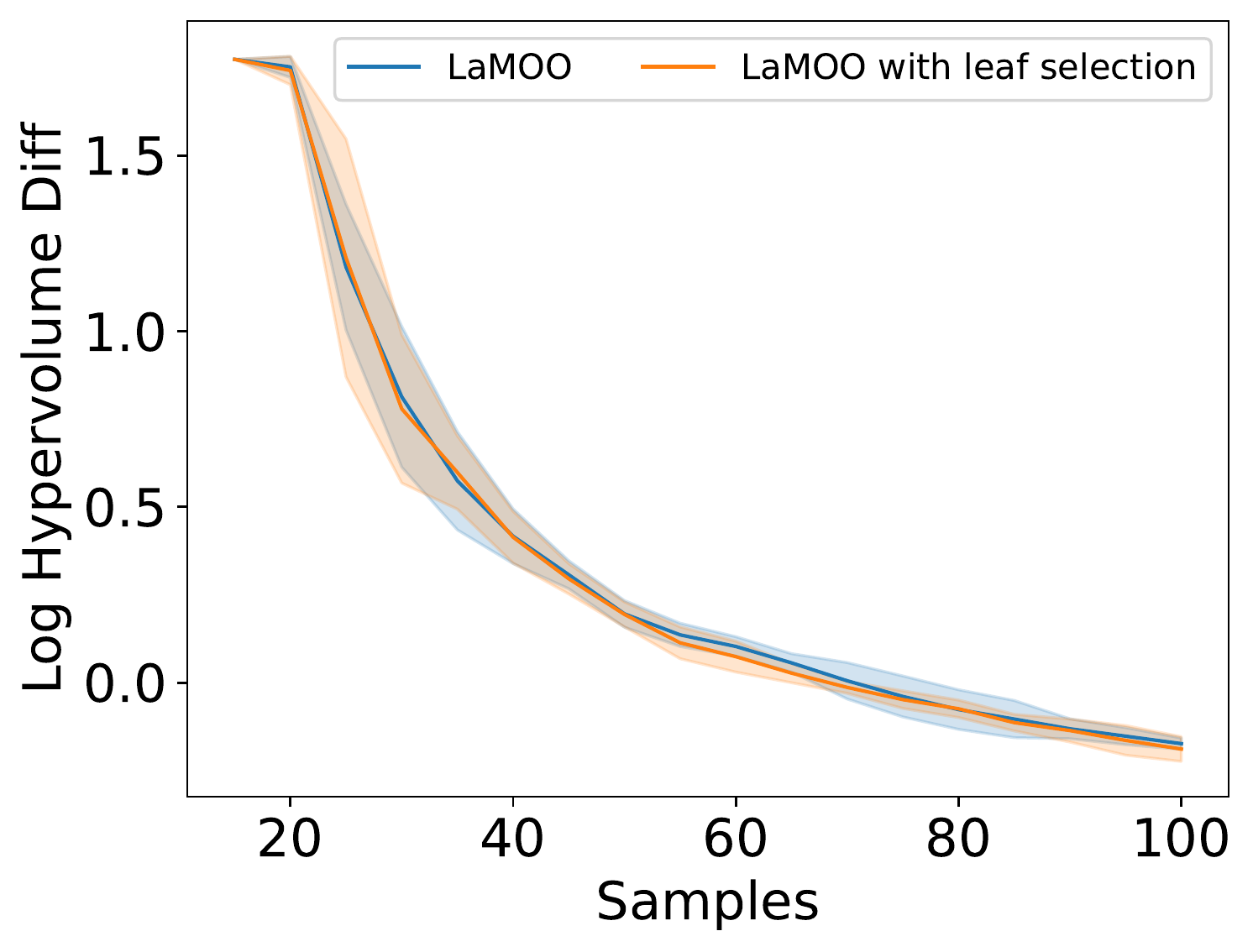}}
\subfloat[][VehicleSafety]{\includegraphics[height=1.33in]{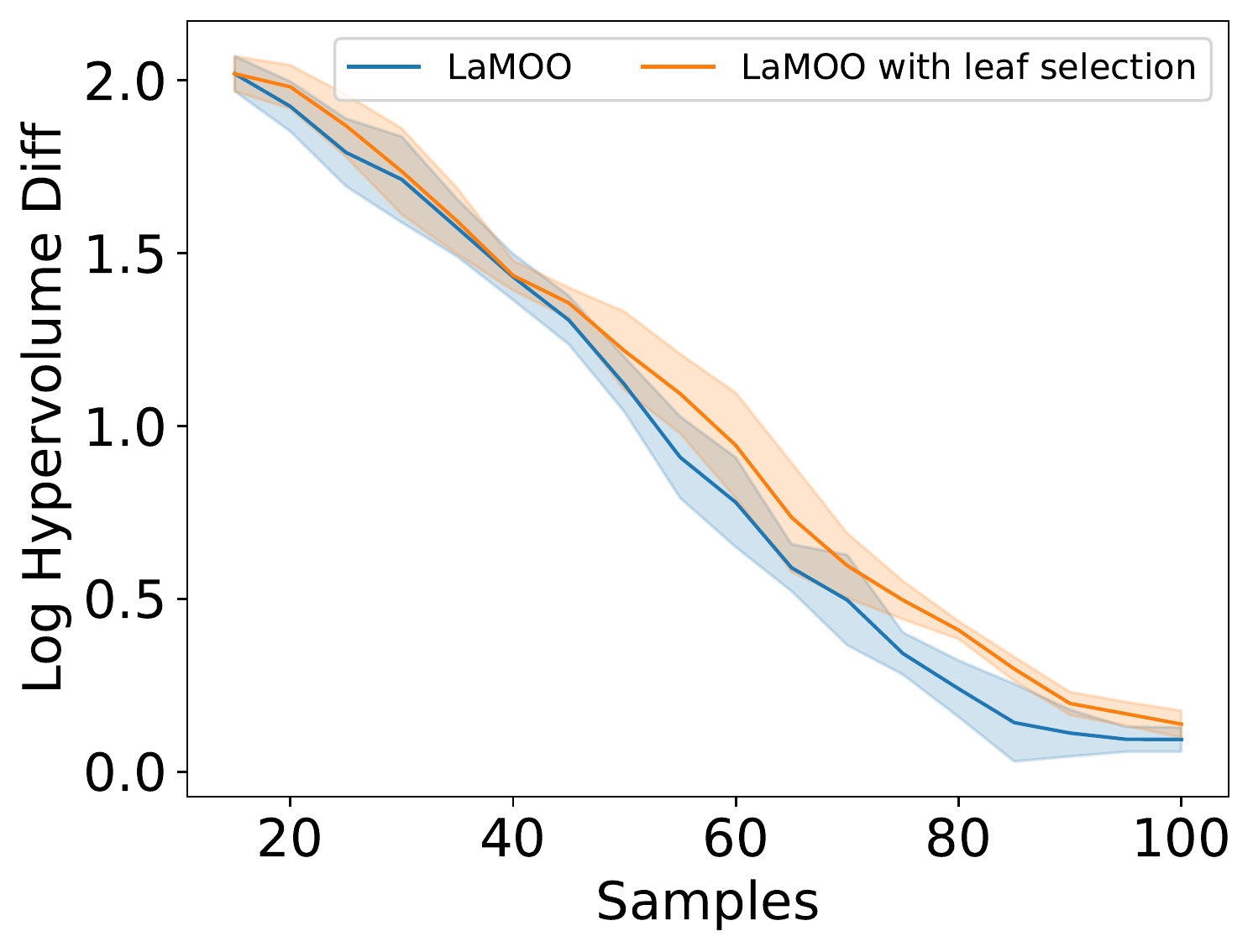}}  
\subfloat[][Nasbench201]{\includegraphics[height=1.33in]{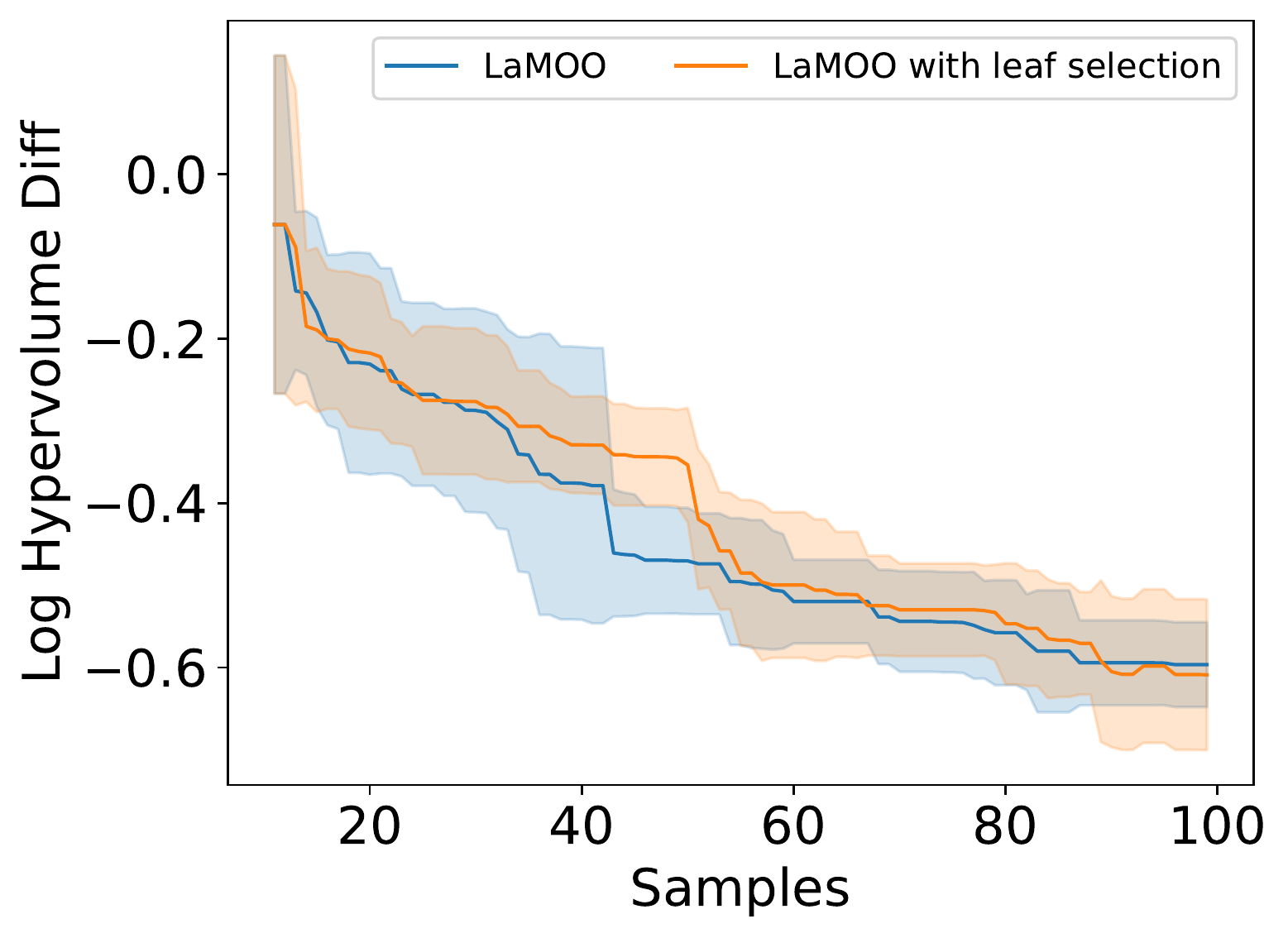}}  
\quad
\caption{ \small Search progress with sample}
\label{fig:leaf_sample}
\end{figure}

\begin{figure}[H]

\centering 
\subfloat[][BraninCurrin]{\includegraphics[height=1.33in]{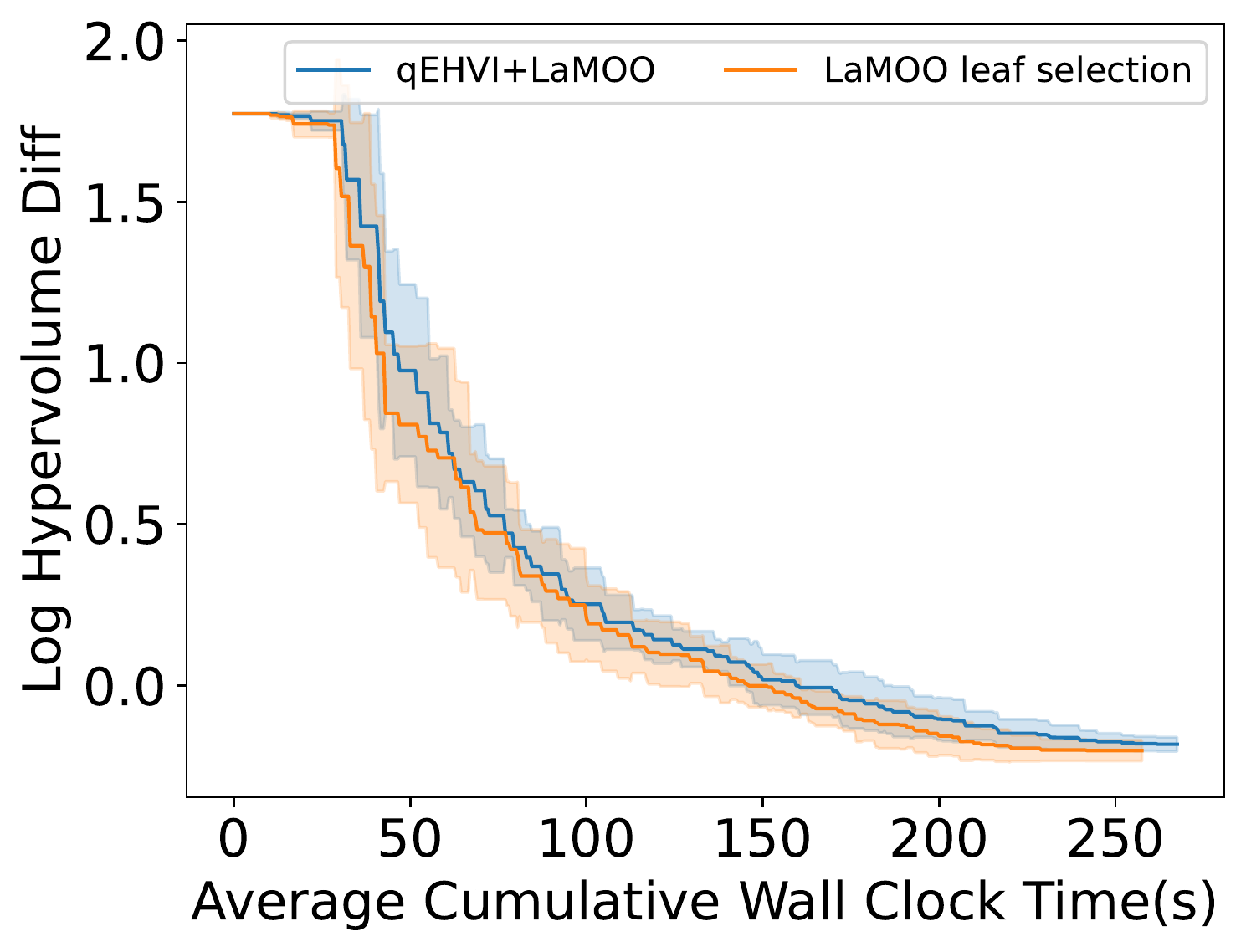}}
\subfloat[][VehicleSafety]{\includegraphics[height=1.33in]{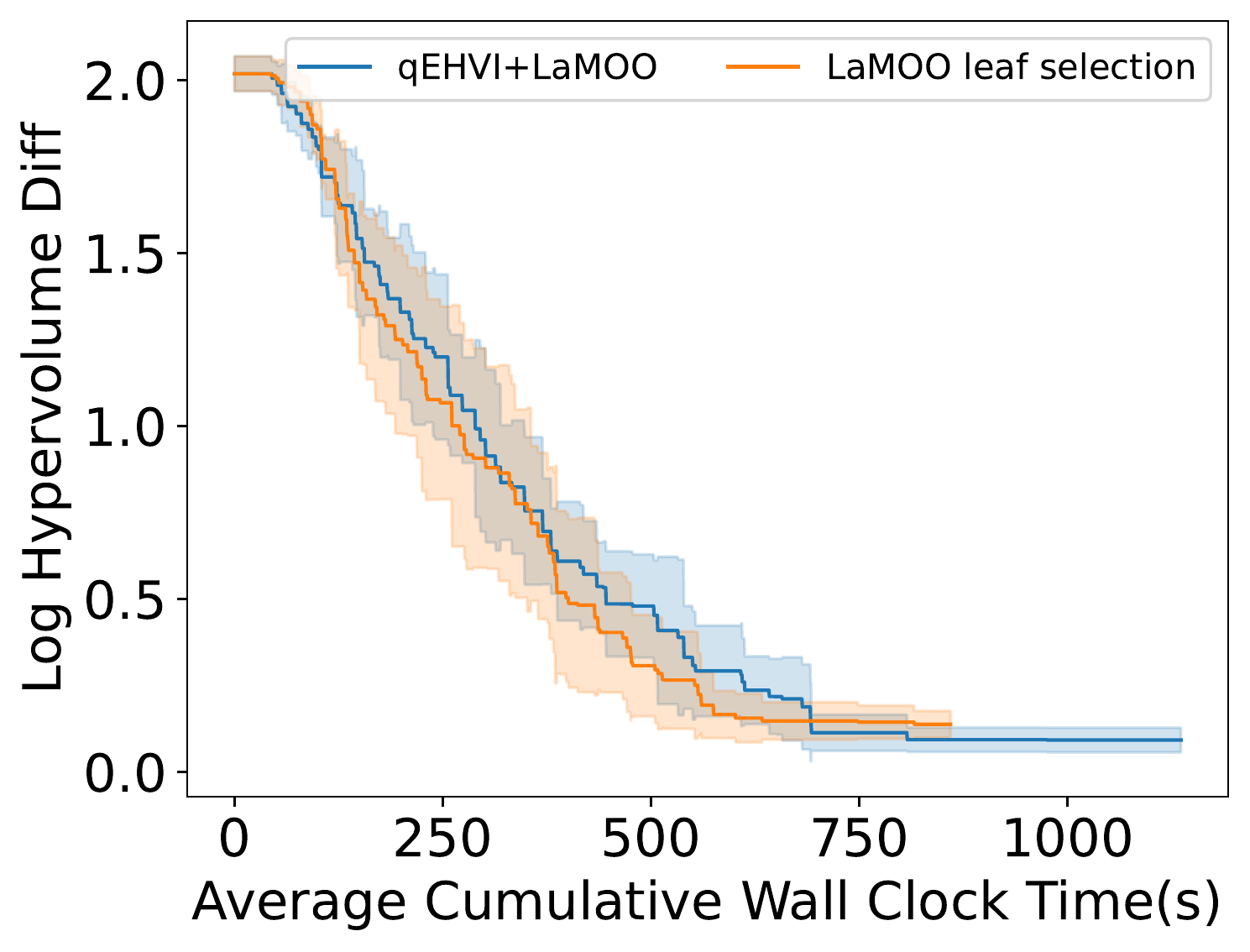}}  
\subfloat[][Nasbench201]{\includegraphics[height=1.33in]{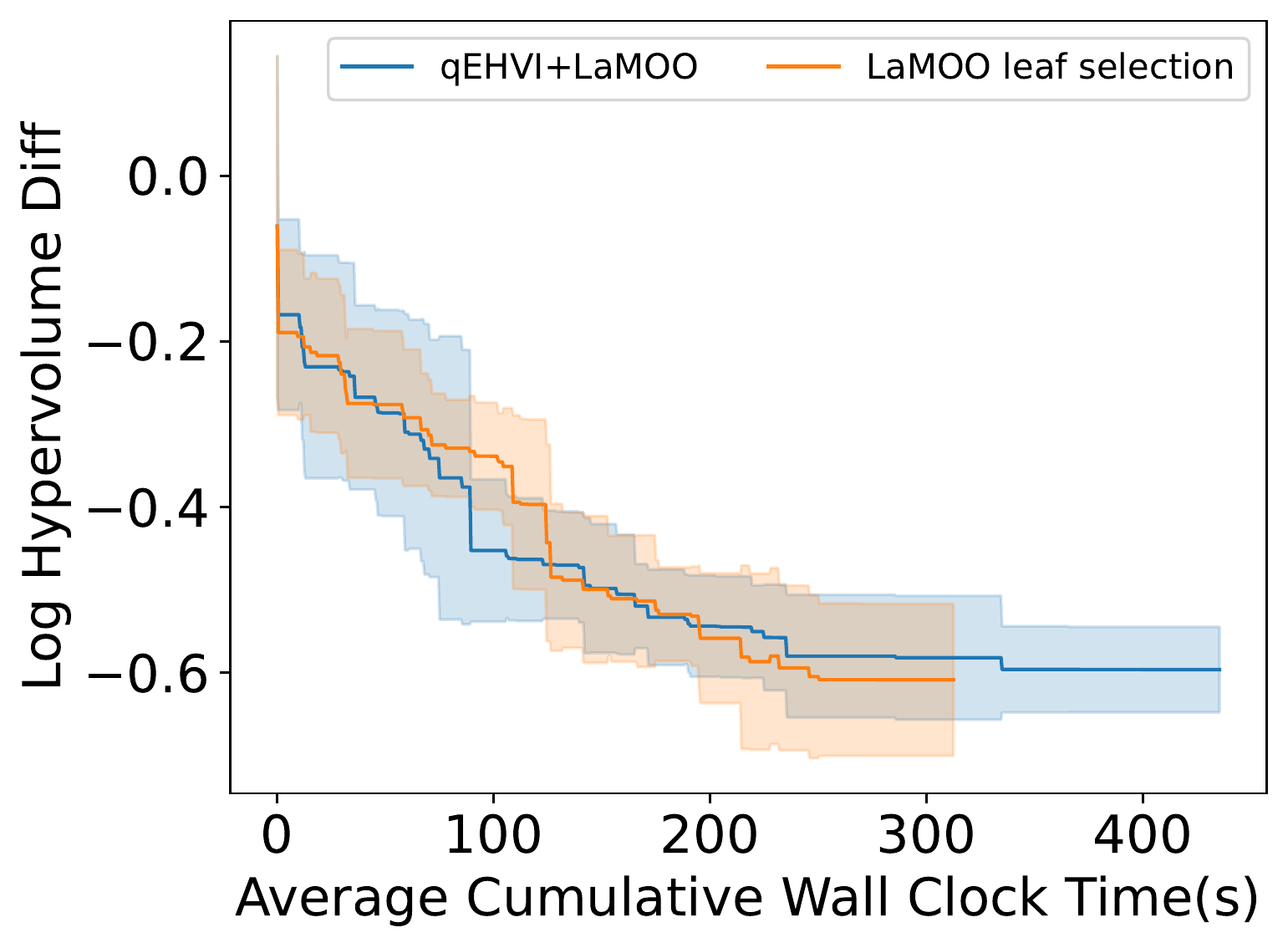}}  
\quad
\caption{ \small Search progress with time}
\label{fig:leaf_time}
\end{figure}

Instead of traversing down the search tree to trace the current most promising search path, this variation of ~\ours{} directly select the leaf node with the highest \emph{UCB value}. Algorithm.~\ref{alg:lamoo_leaf} illustrates the detail of this variation. Therefore, this variation avoids calculating the hypervolume in the non-leaf nodes of the tree, where hypervolume calculation is the main computational cost of ~\ours{} especially in many-objective problems. Figure.~\ref{fig:leaf_sample} and figure.~\ref{fig:leaf_time} validate the variation that is able to reach similar performance of searched samples but saves lots of time. We leave the validation of others problems in the future works. 

\section{additional related works}

~\citep{addr1, addr2, addr3} indeed leverage classifiers to partition search space and draw the new samples from the good region. However, ~\citep{addr1, addr2, addr3} randomly sampled in selected regions without integrating existing optimizers (e.g. Bayesian optimization, evolutionary algorithms). In addition, they progressively select the good regions without a trade-off of exploration and exploitation as we did by leveraging Monte Carlo Tree Search (MCTS). ~\citep{addr4, addr5, addr6} can be seen as the first work to use MCTS to build hierarchical space partitions. But their partitions are predefined (e.g., Voronoi graph, axis-aligned partition, etc) without learning (or adapting to) observed samples so far, except for ~\citep{lanas, wang2020learning, plalam}, which are learning extensions coupled with MCTS. However, they all deal with single-objective optimization. 

For multi-objective optimization, ~\citep{addr7, addr8, addr9} learns to predict the dominance rank of samples, without computing them algorithmic-ally, a slow process with many previous samples, in order to speed up the MOEAs algorithms. Unlike our paper, they do not partition the search space into good/bad regions. In contrast, \ours{} computes the rank algorithmic-ally. Therefore, our contributions are complementary to theirs. We leave a combination of both as one of the future works.

\ours{} V.S. LaMCTS/LaNAS: First, the mechanism of the partitioning of the search space is different. \ours{} uses dominance rank to separate good from bad regions, while LaMCTS uses a k-mean for region separation. LaNAS is even more simple: it uses the median from the single objectives of currently collected samples and a linear classifier to separate regions.

\ours{} V.S. LaP$^3$: LAP$^3$ is a planning algorithm tailored to RL with a single objective function. LAP$^3$ also utilizes the representation learning for the partition space and planning space, while our \ours{} doesn't.

\end{document}